\newif\iffull
\fulltrue

\documentclass[11pt,twoside]{article} 
\usepackage[utf8]{inputenc} 
\usepackage[T1]{fontenc}    
\usepackage{url}            
\usepackage{booktabs}       
\usepackage{nicefrac}       

\usepackage{xspace,amsmath,amssymb,amsfonts,boxedminipage,microtype,hyperref}
\usepackage{natbib}
\usepackage{paralist}
\usepackage{bm}
\usepackage{bbm}
\usepackage{array}
\usepackage{multirow}
\usepackage{amsthm}
\newcolumntype{M}[1]{>{\centering\arraybackslash}m{#1}}
\newcolumntype{N}{@{}m{0pt}@{}}
\usepackage{graphicx}
\usepackage[noend]{algorithmic}
\usepackage{algorithm}
\usepackage{xcolor}
\usepackage{enumitem}

\usepackage{fullpage}

\usepackage{times}

\newcommand{\approxerr}{\varepsilon_{\mathrm{\tiny approx}}}
\newcommand{\generr}{\eps_{\mathrm{\tiny gen}}}
\newcommand{\opterr}{\eps_{\mathrm{\tiny opt}}}
\newcommand{\bzero}{\mathbf{0}}

\newcommand{\RR}{\mathbb{R}}

\newcommand{\PP}{\mathbb{P}}
\newcommand{\EE}{\mathbb{E}}
\newcommand{\eps}{\varepsilon}

\newcommand{\bk}{\mathbf{k}}

\newcommand{\bS}{\mathbf{S}}
\newcommand{\bG}{\mathbf{G}}

\newcommand{\bI}{\mathbf{i}}
\newcommand{\bu}{\mathbf{u}}
 
\newcommand{\bg}{\mathbf{g}}

\newcommand{\bz}{\mathbf{z}}
\newcommand{\br}{\mathbf{r}}

\newcommand{\B}{\mathcal{B}}

\newcommand{\cO}{\mathcal{O}}
\newcommand{\D}{\mathcal{D}}

\newcommand{\X}{\mathcal{X}}
\newcommand{\F}{\mathcal{F}}
\newcommand{\N}{\mathcal{N}}
\newcommand{\Unif}{\mbox{Unif}}

\newcommand{\proj}{\mathsf{Proj}}

\newcommand{\epsrisk}{\eps_{\mbox{\footnotesize risk}}}
\newcommand{\epsopt}{\eps_{\mbox{\footnotesize opt}}}
\newcommand{\epsgen}{\eps_{\mbox{\footnotesize gen}}}

\newcommand{\barx}{\overline{x}}

\newcommand{\pr}[2]{\underset{#1}{\mathbb{P}}\left[ #2 \right]}
\newcommand{\ex}[2]{\underset{#1}{\mathbb{E}}\left[ #2 \right]}

\newcommand{\risk}{\varepsilon_{\mathsf{\footnotesize{risk}}}}
\newcommand{\Ansgd}{\A_{\sf NSGD}}

\usepackage{vfmacros}

\usepackage{ifthen}
\newboolean{fullver}

\title{Stability of
Stochastic Gradient Descent \\
on Nonsmooth Convex Losses}
\author{Raef Bassily\thanks{Department of Computer Science \& Engineering, The Ohio State University. \texttt{bassily.1@osu.edu}}
	\and Vitaly Feldman\thanks{Work done while at Google Research. \texttt{vitaly.edu@gmail.com}}  \and
Crist\'{o}bal Guzm\'{a}n \thanks{Institute for Mathematical and Computational Engineering, Pontificia Universidad Cat\'olica de Chile. \texttt{crguzmanp@mat.uc.cl}}
\and Kunal Talwar\thanks{Work done while at Google Research. \texttt{kunal@kunaltalwar.org}}
}
\date{}

\newcommand{\UHS}{UAS\xspace}
\newcommand{\stabname}{uniform argument stability\xspace}
\newcommand{\Stabname}{Uniform Argument Stability}

\begin{document}

\maketitle

\begin{abstract}
Uniform stability is a notion of algorithmic stability that bounds the worst case change in the model output by the algorithm when a single data point in the dataset is replaced. An influential work of \citet{HardtRS16} provides strong upper bounds on the uniform stability of the stochastic gradient descent (SGD) algorithm on sufficiently smooth convex losses. These results led to important progress in understanding of the generalization properties of SGD and several applications to differentially private convex optimization for smooth losses.

Our work is the first to address uniform stability of SGD on {\em nonsmooth} convex losses. Specifically, we provide sharp upper and lower bounds for several forms of SGD and full-batch GD
on arbitrary Lipschitz nonsmooth convex losses. Our lower bounds show
that, in the nonsmooth case, (S)GD can be inherently less stable than in the smooth case. On the other hand, our upper bounds show that (S)GD is sufficiently stable for deriving new and useful bounds on generalization error. Most notably, we obtain the first dimension-independent generalization bounds for multi-pass SGD in the nonsmooth case. 
In addition, our bounds allow us to derive a new algorithm for differentially private nonsmooth stochastic convex optimization with optimal excess population risk. Our algorithm is simpler and more efficient than the best known algorithm for the nonsmooth case \citep{FKT19}.
\end{abstract}

\section{Introduction}
Successful applications of a machine learning algorithm require the algorithm to generalize well to unseen data. Thus
understanding and bounding the generalization error of machine learning algorithms is an area of intense theoretical interest and practical importance. The single most popular approach to modern machine learning relies on the use of continuous optimization techniques to optimize the appropriate loss function, most notably the stochastic (sub)gradient descent (SGD) method. Yet the generalization properties of SGD are still not well understood.  

Consider the setting of stochastic convex optimization (SCO). In this problem, we are interested in the minimization of the population risk $F_{\D}(x):=\EE_{\bz\sim \D}[ f(x,\bz) ]$, where $\D$ is an arbitrary and unknown distribution, for which we have access to an i.i.d.~sample of size $n$, $\bS=(\bz_1,\ldots,\bz_n)$; and $f(\cdot,z)$ is convex and Lipschitz for all $z$.
The performance of an
algorithm ${\cal A}$ is
quantified by its expected {\em excess population risk},
$$ \varepsilon_{\mbox{\footnotesize{risk}}}(\A) :=
\EE[F_{\D}(\A(\bS))]- \min_{x\in \X} F_{\D}(x), $$ where the expectation is taken with respect to the randomness of the sample $\bS$ and internal randomness of $\A$.
A standard way to bound the excess risk is given by its decomposition into optimization error (a.k.a.~training error) and generalization error (see eqn.~\eqref{eqn:risk_decomp} in Sec.~\ref{sec:prelims}). The optimization error can be easily measured empirically but assessing the generalization error requires access to fresh samples from the same distribution. Thus bounds on the generalization error lead directly to provable guarantees on the excess population risk.

Classical analysis of SGD allows obtaining bounds on the excess population risk of one pass SGD. In particular, with an appropriately chosen step size, SGD gives a solution with expected excess population risk of $O(1/\sqrt{n})$ and this rate is optimal \citep{nemirovsky1983problem}. However, this analysis does not apply to multi-pass SGD that is ubiquitous in practice.


In an influential work, \citet{HardtRS16} gave the first bounds on the generalization error of general forms of SGD (such as those that make multiple passes over the data). Their analysis relies on algorithmic stability, a classical tool for proving bounds on the generalization error. Specifically, they gave strong bound on the {\em uniform stability} of several variants of SGD on convex and smooth losses (with $2/\eta$-smoothness sufficing when all the step sizes are at most $\eta$). Uniform stability bounds the worst case change in loss of the model output by the algorithm on the worst case point when a single data point in the dataset is replaced \citep{BousquettE02}.  Formally, for a randomized algorithm ${\cal A}$, loss functions $f(\cdot,z)$ and $S\simeq S^{\prime}$ and $z\in\Z$, let $\gamma_{\cal A}(S,S^{\prime}, z):=f(\A(S),z)-f(\A(S^{\prime}),z)$, where $S\simeq S^{\prime}$ denotes that the two datasets differ only in a single data point. We say ${\cal A}$ is $\gamma$-uniformly stable if
$$\sup_{S\simeq S^{\prime},z}\mathbb{E}[\gamma_{\cal A}(S,S^{\prime},z)]\leq \gamma,$$
where the expectation is over the internal randomness of $\cal{A}.$ Stronger notions of stability can also be considered, e.g., bounding the probability -- over the internal randomness of $\cal{A}$ -- that $\gamma_{\cal A}(S,S^{\prime},z)>\gamma$. 
Using stability,
\cite{HardtRS16} showed that several variants of SGD simultaneously achieve the optimal tradeoff between the excess empirical risk and stability with both being $O(1/\sqrt n)$. Several works have used this approach to derive new generalization properties of SGD \citep{London:2017,Chen:2018,FeldmanV:2019}.

The key insight of \citet{HardtRS16} is that a gradient step on a sufficiently smooth convex function is a nonexpansive operator (that is, it does not increase the $\ell_2$ distance between points). Unfortunately, this property does not hold for nonsmooth losses such as the hinge loss. As a result, no non-trivial bounds on the uniform stability of SGD have been previously known in this case.

Uniform stability is also closely related to the notion of differential privacy (DP). DP upper bounds the worst case change in the output distribution of an algorithm when a single data point in the dataset is replaced \citep{DMNS06}. This connection has been exploited in the design of several DP algorithms for SCO. In particular, bounds on the uniform stability of SGD from \citep{HardtRS16} have been crucial in the design and analysis of new DP-SCO algorithms \citep{wu2017bolt,DworkFeldman18,FeldmanMTT18,BFTT19,FKT19}.


\subsection{Our Results}
We establish tight bounds on the uniform stability of the (stochastic) subgradient descent method on nonsmooth convex losses. These results demonstrate that in the nonsmooth case SGD can be substantially less stable. At the same time we show that SGD has strong stability properties even in the regime when its iterations can be expansive.

For convenience, we describe our results in terms of {\em \stabname} (\UHS), which bounds the output sensitivity in $\ell_2$-norm w.r.t.~an arbitrary change in a single data point. Formally, a (randomized) algorithm has $\delta$-\UHS if 
\begin{equation} \label{eqn:UAS_intro}
\sup_{S\simeq S^{\prime}} \EE\left\| \A(S)-\A(S^{\prime})\right\|_2 \leq \delta.
\end{equation}
This notion is implicit in existing analyses of uniform stability
\citep{BousquettE02,ShwartzSSS10,HardtRS16} and was explicitly defined by \citet{Liu:2017}. 
In this work, we prove stronger -- high probability -- upper bounds on the random variable $\delta_{\cal A}(S,S^{\prime}):=\left\|\A(S)-\A(S^{\prime})\right\|$,\footnote{In fact, for both GD and fixed-permutation SGD we can obtain w.p.~1 upper bounds on $\delta_{\cal A}(S,S^{\prime})$, whereas for sampling-with-replacement SGD, we obtain a high-probability upper bound.} and we provide matching lower bounds for the weaker -- in expectation -- notion of \UHS \eqref{eqn:UAS_intro}.
A summary of our bounds is in Table \ref{tab:summary}. For simplicity, they are provided for constant step size; general step sizes (for upper bounds) are provided in Section \ref{sec:UB}.
\begin{table}[h!]
\centering
\begin{tabular}{| M{2.8cm} | M{3cm} | M{4cm} | M{4.5cm}  | N}
\hline
Algorithm	 & H.p.~upper bound    & Exp.~upper bound  & Exp.~Lower bound &\\[5pt] \hline \small
 GD (full batch)  & $4\big(\eta\sqrt{T}+\frac{\eta T}{n} \big)$ &
 $4\big(\eta\sqrt{T}+\frac{\eta T}{n} \big)$
&   $\Omega\big(\eta\sqrt T+ \frac{\eta T}{n}\big)$       %
&\\[6pt]\hline
 \small SGD (w/replacement) &
 $4\big(\eta\sqrt{T}+\frac{\eta T}{n}\big) $
 &$\min\{1,\frac{T}{n}\}4\eta\sqrt{T}+4\frac{\eta T}{n} $
& $\Omega\Big(\min\{1,\frac{T}{n}\}\eta\sqrt{T}+\frac{\eta T}{n}\Big)$ &\\[5pt] \hline
\small SGD (fixed permutation) &
$2\eta\sqrt{T}+4\frac{\eta T}{n}$& $\min\{1,\frac{T}{n}\}2\eta\sqrt{T}+4\frac{\eta T}{n}$
& $\Omega\Big(\min\{1,\frac{T}{n}\}\eta\sqrt{T}+\frac{\eta T}{n}\Big)$ &\\[5pt] \hline
\end{tabular}
\caption{
\iffull
\UHS for variants of GD/SGD, with normalized radius and Lipschitz constant. Here $T$ is the number of iterations and $\eta>0$ is the step size. 
Both upper and lower bounds also are $\min\{2,(\cdot)\}$, due to the feasible domain radius.
\fi
}
\label{tab:summary}
\end{table}

Compared to the smooth case \citep{HardtRS16}, the main difference is the presence of the additional $\eta\sqrt{T}$ term. This term has important implications for the generalization bounds derived from \UHS. The first one is that the standard step size $\eta=\Theta(1/\sqrt{n})$ used in single pass SGD leads to a vacuous stability bound. Unfortunately, as shown by our lower bounds, this  is unavoidable (at least in high dimension). However, by decreasing the step size and increasing the number of steps, one obtains a variant of SGD with nearly optimal balance between the \UHS and the excess empirical risk. 

We highlight two major consequences of our bounds: 
\begin{itemize}[leftmargin=*]
\item \textbf{Generalization bounds for multi-pass nonsmooth SGD.} We prove that the generalization error of multi-pass SGD  with $K$ passes is bounded by $O((\sqrt{Kn}+K)\eta)$. This result can be easily combined with training error guarantees to provide excess risk bounds for this algorithm. Since training error can be measured directly, our generalization bounds would immediately yield strong guarantees on the excess risk in practical scenarios where we can certify small training error.


\item \textbf{Differentially private stochastic convex optimization for non-smooth losses.} We show that a variant of standard noisy SGD \citep{BST14} with constant step size and $n^2$ iterations yields the optimal excess population risk $O\big(\frac{1}{\sqrt{n}}+\frac{\sqrt{d\log (1/\beta)}}{\alpha n}\big)$ for convex \emph{nonsmooth} losses under $(\alpha,\beta)$-differential privacy. The best previous algorithm for this problem is substantially more involved: it relies on a multi-phase regularized SGD with decreasing step sizes and variable noise rates and uses $O(n^2 \sqrt{\log (1/\beta)})$ gradient computations \citep{FKT19}.

\end{itemize}

\subsection{Overview of Techniques}

\begin{itemize}[leftmargin=*]
\item {\bf Upper bounds.}
When gradient steps are nonexpansive, upper-bounding \UHS requires simply summing the differences between the gradients on the neighboring datasets when the replaced data point is used 
\citep{HardtRS16}. This gives the bound of $\eta T/n$ in the smooth case.


By contrast, in the nonsmooth case, \UHS may increase even when the gradient step is performed on the same function. As a result it may increase in {\em every single iteration}. However, we use the fact  that the difference in the subgradients has negative inner product with the difference between the iterates themselves (by monotonicity of the subgradient). Thus the increase in distance satisfies a recurrence with a quadratic and a linear term. Solving this recurrence leads to our upper bounds.

\item {\bf Lower bounds.}
The lower bounds are based on a function with a highly nonsmooth behavior around the origin. More precisely, it is the maximum of linear functions
plus
a small linear drift that is controlled by a single data point. We show that, when starting the algorithm from the origin, the presence of the linear drift pushes the iterate into a trajectory in which each subgradient step is orthogonal to the current iterate. Thus, if $d \geq \min\{T,1/\eta^2\}$,  we get the $\sqrt{T}\eta$ increase in \UHS. Our lower bounds are also robust to averaging of the iterates. 
\end{itemize}

\subsection{Other Related Work}

Stability is a classical approach to proving generalization bounds pioneered by \citet{RogersWagner78,DevroyeW79,DevroyeW79a}. It is based on analysis of the sensitivity of the learning algorithm to changes in the dataset such as leaving one of the data points out or replacing it with a different one. \iffull The choice of how to measure the effect of the change and various ways to average over multiple changes give rise to a variety of stability notions 
(e.g., \citep{BousquettE02,MukherjeeNPR06,ShwartzSSS10}).\fi\,\! Uniform stability was introduced by \citet{BousquettE02} in order to derive general bounds on the generalization error that hold with high probability. These bounds have been significantly improved in a recent sequence of works \citep{FeldmanV:2018,FeldmanV:2019,BousquetKZ19}. \iffull A long line of work focuses on the relationship between various notions of stability and learnability in supervised setting (see  \citep{Kearns:1999,PoggioRMN04,ShwartzSSS10} for an overview). These works employ relatively weak notions of average stability and derive a variety of asymptotic equivalence results. \fi
\citet{Chen:2018} establish limits of stability in the smooth convex setting, proving that accelerated methods must satisfy strong stability lower bounds. Stability-based data-dependent generalization bounds for continuous losses were studied in \citep{Maurer17,Kuzborskij:2018}.

First applications of uniform stability in the context of stochastic convex optimization relied on the stability of the empirical minimizer for strongly convex losses \citep{BousquettE02}. Therefore a natural approach to achieve uniform stability (and also \UHS) is to add a strongly convex regularizer and solve the ERM to high accuracy \citep{ShwartzSSS10}. Recent applications of this approach can be found for example in \citep{KorenLevy15,Charles:2018,FKT19}. In contrast, our approach does not require strong convexity and applies to all iterates of the SGD and not only to a very accurate empirical minimizer.

Classical approach to generalization relies on \emph{uniform convergence} of empirical risk to population risk.
Unfortunately, without additional structural assumptions on convex functions, a lower bound of $\Omega(\sqrt{d/n})$ on the rate of \emph{uniform convergence} for convex SCO is known \citep{ShwartzSSS10,Feldman:16erm}. The dependence on the dimension $d$ makes the bound obtained via the uniform-convergence approach vacuous in the high-dimensional settings common in modern applications.

Differentially private convex optimization has been studied extensively for over a decade (see, e.g., \citep{CM08,CMS,jain2012differentially,kifer2012private,ST13sparse, BST14,ullman2015private,JTOpt13,talwar2015nearly, BFTT19, FKT19}).
However, until recently, the research focused on minimization of the empirical risk. \iffull Population risk for DP-SCO was first studied by \citet{BST14} who gave an upper bound of $\max\left(\tfrac{d^{\frac{1}{4}}}{\sqrt{n}}, \tfrac{\sqrt{d}}{\alpha n}\right)$ \cite[Sec. F]{BST14} on the excess risk. \fi A recent work of  \citet{BFTT19} established that the optimal rate of the excess population risk for $(\alpha,\beta)$-DP SCO algorithms is $O\big(\frac{1}{\sqrt{n}}+\frac{\sqrt{d\log (1/\beta)}}{\alpha n}\big)$. Their algorithms are relatively inefficient, especially in the nonsmooth case. Subsequently, \citet{FKT19} gave several new algorithms for DP-SCO with the optimal population risk. For sufficiently smooth losses, their algorithms use a linear number of gradient computations. In the nonsmooth case, as mentioned earlier, their algorithm requires $O(n^2\sqrt{\log(1/\beta)})$ gradient computations and is significantly more involved than the algorithm shown here.

\section{Notation and Preliminaries}\label{sec:prelims}

Throughout we work on the Euclidean space $(\RR^d,\|\cdot\|_2)$. \iffull Therefore, we use unambiguously $\|\cdot\|=\|\cdot\|_2$. Vectors are denoted by lower case letters, e.g.~$x,y$. Random variables (either scalar or vector) are denoted by boldface letters, e.g.~$\bz,\bu$.\fi We denote the Euclidean ball of radius $r>0$ centered at $x\in\RR^d$ by $\B(x,r)$.
\iffull In what follows, $\X\subseteq \RR^d$ is a compact convex set, and assume we know its Euclidean radius $R>0$, $\X\subseteq \B(0,R)$. \fi \iffull Let $\proj_{\X}$ be the Euclidean projection onto $\X$, which is {\em nonexpansive}  $\|\proj_{\X}(x)-\proj_{\X}(y)\|\leq\|x-y\|$. \fi
A convex function $f:\X\mapsto \RR$ is $L$-Lipschitz
if
\iffull
\begin{equation} \label{eqn:Lipschitz}
 f(x)-f(y) \leq L\|x-y\| \qquad(\forall x,y\in\X).
\end{equation}
Functions with these properties are guaranteed to be subdifferentiable. Moreover,
in the convex case, property~\eqref{eqn:Lipschitz} is ``almost'' equivalent to having subgradients bounded as $\partial f(x)\subseteq \B(0,L)$, for all $x \in\X$.\footnote{For equivalence to hold it is necessary that the function is well-defined and satisfies \eqref{eqn:Lipschitz} over an open set containing $\X$, see Thm.~3.61 in \cite{Beck:2017}. We will assume this is the case, which can be done w.l.o.g..}
\fi
We denote the class of convex $L$-Lipschitz functions as $\F_{\X}^0(L)$.
\iffull
With slight abuse of notation, given a function $f\in\F_{\X}^0(L)$, we will denote by $\nabla f(x)$ an arbitrary choice of $g\in\partial f(x)$.
\fi
In this work, we will focus on the class $\F_{\X}^0(L)$ defined over a compact convex set $\X$. Since the Euclidean radius of $\X$ is bounded by $R$, we will assume that the range of these functions lies in $[-RL, RL]$.

\ifthenelse{\boolean{fullver}}{
A convex and differentiable function $f:\X\mapsto\RR$ is said to be $\mu$-smooth if
$$ \|\nabla f(x)-\nabla f(y)\| \leq \mu\|x-y\| \qquad(\forall x,y\in\X),$$
and we denote the class of convex $\mu$-smooth functions by $\F_{\X}^1(\mu)$.
}{}

\noindent\textbf{Nonsmooth stochastic convex optimization:} We study the standard setting of nonsmooth stochastic convex optimization
\begin{equation}\nonumber
x^{\ast} \in \arg\min \{F_{\D}(x):=\EE_{\bz\sim \D}[ f(x,\bz) ]:\,\, x\in \X \}.
\end{equation}
Here,  $\D$ is an unknown distribution supported on a set $\Z$, and $f(\cdot,z)\in\F_{\X}^0(L)$ for all $z\in\Z$. In the stochastic setting, we assume access to an i.i.d.~sample from $\D$, denoted as $\bS=(\bz_1,\ldots, \bz_n)\sim \D^n$. \iffull Here, we will use the bold symbol $\bS$
to denote a random sample from the unknown distribution. A fixed (not random) dataset from $\Z^n$ will be denoted as $S=(z_1,\ldots, z_n)\in\Z^n$.\fi

\noindent\textbf{A stochastic optimization algorithm} is a (randomized) mapping ${\cal A}:\Z^n\mapsto \X$. When the algorithm is randomized, ${\cal A}(\bS)$ is a random variable depending on both the sample $\bS\sim \D^n$ and its own random coins.
The performance of ${\cal A}$ is quantified by its {\em excess population risk}
$$ \varepsilon_{\mbox{\footnotesize{risk}}}(\A) :=
F_{\D}(\A(\bS))-F_{\D}(x^{\ast}).$$
Note that $\varepsilon_{\mbox{\footnotesize{risk}}}(\A)$ is a random variable (due to randomness in the sample $\bS$ and any possible internal randomness of the algorithm). \iffull Our guarantees on the excess population risk will be expressed in terms of upper bounds on this quantity that hold \textit{with high probability} over the randomness of both $\bS$ and the random coins of the algorithm.\fi

\noindent\textbf{Empirical risk minimization (ERM)} is one of the most standard approaches to stochastic convex optimization. In the ERM problem, we are given a sample  $\bS=(\bz_1,\ldots, \bz_n)$, and the goal is to find
$$x^{\ast}(\bS) \in \arg\min \Big\{F_{\bS}(x):=\frac1n \sum_{i=1}^n f(x, \bz_i):\,\, x\in \X \Big\}.$$

\iffull One way to bound the excess population risk is to solve the ERM problem, and appeal to uniform convergence; however, uniform convergence
rates in this case are dimension-dependent, $\Omega(\sqrt{d/n})$ \citep{Feldman:16erm}.\fi

\noindent\textbf{Risk decomposition:} \iffull Guaranteeing low excess population risk for a general algorithm is a nontrivial task. \fi A common way
to bound \iffull it \fi is by decomposing it into {\em generalization}, {\em optimization} and {\em approximation error}:
\begin{equation}
    \label{eqn:risk_decomp}
\varepsilon_{\mbox{\footnotesize{risk}}}(\A)
\leq \underbrace{F_{\D}(\A(\bS))-F_{\bS}(\A(\bS))}_{\eps_{\mbox{\tiny gen}}(\A)}
+\underbrace{F_{\bS}(\A(\bS))-F_{\bS}(x^{\ast}(\bS))}_{\eps_{\mbox{\tiny opt}}(\A)}
+\underbrace{F_{\bS}(x^{\ast}(\bS))-F_{\D}(x^{\ast})}_{\varepsilon_{\mbox{\tiny approx}}}.
\end{equation}
Here, the optimization error corresponds to the empirical optimization gap, which can be bounded by standard optimization convergence analysis. The expected value of the approximation error is at most zero. \iffull One can show, e.g., by Hoeffding's inequality, that the approximation error is bounded by $\tilde{O}(LR/\sqrt n)$ with high probability (see Lemma~\ref{lem:approx-err-bd} below.) \fi Therefore, to establish bounds on the excess risk it suffices to upper bound the optimization and generalization errors.
\iffull
\begin{lem}\label{lem:approx-err-bd}
For any $\theta \in (0, 1),$ with probability at least $1-\theta$, the approximation error is bounded as
$$\approxerr \leq \frac{RL \sqrt{2 \log(1/\theta)}}{\sqrt{n}}.$$
\end{lem}
\begin{proof}
First, note that $F_{\D}(x^{\ast})=\ex{\bS}{F_S(x^{\ast})}=\frac1n\sum_{i=1}^nf(x^{\ast},\bz_i)$. Hence, by independence and the fact that $f(x^{\ast},\bz_i) \in [-RL, RL]$ with probability $1$ for all $i\in [n]$, the following follows from Hoeffding's inequality:
$$\pr{\bS\sim\D^n}{F_{\bS}(x^{\ast})-F_{\D}(x^{\ast})\geq \frac{RL \sqrt{2 \log(1/\theta)}}{\sqrt{n}}}\leq \theta.$$
Finally, note that by definition of $x^{\ast}(\bS)$, we have $F_{\bS}(x^{\ast}(\bS))-F_{\bS}(x^{\ast})\leq 0$. Combining this with the above bound completes the proof.


\end{proof}
\fi

We say that two datasets $S, ~S^{\prime}$ are neighboring, denoted $S\simeq S^{\prime}$, if they only differ on a single entry; i.e., there exists $i\in[n]$ s.t.~for all $k\neq i$, $z_k=z_k^{\prime}$.


\noindent\textbf{Uniform argument stability (UAS):} Given an algorithm ${\cal A}$ and datasets $S\simeq S^{\prime}$, we define the {\em uniform argument stability} (UAS) random variable as
$$\delta_{\cal A}(S,S^{\prime}):=\|{\cal A}(S)-{\cal A}(S^{\prime})\|.$$
The randomness here is due to any possible internal randomness of ${\cal A}$. For any $L$-Lipschitz function $f$, we have that $f\left(\A(S), z\right)-f\left(\A(S^{\prime}), z\right)\leq L\,\delta_{\cal A}(S,S^{\prime}).$ Hence, upper bounds on UAS can be easily transformed into upper bounds on uniform stability.


In this work, we will consider two types of bounds on UAS.
\iffull
\subsection{High-probability guarantees on UAS}
In Section~\ref{sec:UB}, we give upper bounds on UAS for three variants of the (stochastic) gradient descent algorithm, namely, (i) full-batch gradient descent, (ii) sampling-with-replacement stochastic gradient descent, and (iii) fixed-permutation stochastic gradient descent. Variant (i) is deterministic (and hence UAS is a deterministic quantity). For variant (ii), for any pair of neighboring datasets $S, S'$, we give an upper bound on the UAS random variable that holds with high probability over the algorithm's internal randomness (the sampling with replacement). For variant (iii), we give an upper bound on UAS that holds for an arbitrary choice of permutation; in particular, for any random permutation our upper bound on the UAS random variable that holds with probability 1.

\fi
High-probability upper bounds on UAS lead to high-probability upper bounds on generalization error $\generr$. We will use the following theorem, which follows in a straightforward fashion from \cite[Theorem 1.1]{FeldmanV:2019}, to derive generalization-error guarantees for our results in Sections~\ref{sec:MultipassSGD} and \ref{sec:Applications} based on our UAS upper bounds in Section~\ref{sec:UB}.

\begin{thm}[follows from Theorem 1.1 in \cite{FeldmanV:2019}]\label{thm:stab-to-gen}
Let $\A:\Z^n\rightarrow \X$ be a randomized algorithm. For any pair of neighboring datasets $S, S'$, suppose that the UAS random variable of $\A$ satisfies:
\begin{align*}
    \pr{\A}{\delta_{\A}(S, S')\geq \gamma}&\leq \theta_0.
\end{align*}
Then there is a constant $c$ such that for any distribution $\D$ over $\Z$ and any $\theta \in (0, 1)$, we have
\begin{align*}
    \pr{\bS\sim\D^n,\,\A}{\lvert\generr(\A)\rvert\geq c\left(L \gamma\log(n)\log(n/\theta)+LR\sqrt{\frac{\log(1/\theta)}{n}}\right)}&\leq \theta+\theta_0,
\end{align*}
where $\generr(\A)= F_{\D}(\A(\bS))-F_{\bS}(\A(\bS))$ as defined earlier.
\end{thm}

\iffull
\subsection{Expectation guarantees on UAS}
\fi Our results also include upper and lower bounds on $\sup\limits_{S\simeq S'}\ex{\A}{\delta_{\A}(S, S')}$; that is the supremum of the expected value of the UAS random variable, where the supremum is taken over all pairs of neighboring datasets.
\iffull
In Section~\ref{sec:w-replace-sgd}, we provide an upper bound on this quantity for the sampling-with-replacement stochastic gradient descent. The upper bounds on the other two variants of the gradient descent method hold in the strongest sense (they hold with probability $1$).  Moreover, in Appendix~\ref{app:T_less_n}, we give slightly tighter expectation guarantees on UAS for both sampling-with-replacement SGD and fixed-permutation SGD with a uniformly random permutation.  

In Section~\ref{sec:LowerBounds}, we give lower bounds on this quantity for the two variants of the stochastic subgradient method, together with a deterministic lower bound for the full-batch variant.


\fi 
\iffull
\section{Upper Bounds on \Stabname}
\fi
\label{sec:UB}

\subsection{The Basic Lemma}


We begin by stating a key lemma that encompasses the UAS bound analysis of multiple variants of (S)GD. In particular, all of our UAS upper bounds are obtained by almost a direct application of this lemma.
In the lemma we consider two gradient descent trajectories associated to different sequences of objective functions. The degree of concordance of the two sequences, quantified by the distance between the subgradients at the current iterate, controls the deviation between the trajectories. We note that this distance condition is satisfied for all (S)GD variants we study in this work.

\begin{lem}\label{lem:main_lem}
Let $(x^t)_{t\in[T]}$ and $(y^t)_{t\in[T]}$, with $x^1=y^1$, be online gradient descent trajectories for convex $L$-Lipschitz objectives $(f_t)_{t\in[T-1]}$ and $(f_t^{\prime})_{t\in[T-1]}$, respectively; i.e.,
\begin{eqnarray*}
x^{t+1} &=& \proj_{\X}[x^t-\eta_t \nabla f_t(x^t)]\\
y^{t+1} &=& \proj_{\X}[y^t-\eta_t \nabla f_t^{\prime}(y^t)],
\end{eqnarray*}
for all $t\in[T-1]$. Suppose for every $t\in[T-1]$,
$\|\nabla f_t(x^t)-\nabla f_t^{\prime}(x^t)\|\leq a_t$, 
for scalars $0\leq a_t\leq 2L$. Then, if  $t_0=\inf\{t: f_t\neq f_t^{\prime}\},$
$$ \|x^T-y^T\| \leq 2L\sqrt{\sum_{t=t_0}^{T-1}\eta_t^2}+2\sum_{t=t_0+1}^{T-1}\eta_t a_t.$$
\end{lem}

\begin{proof}
Let 
$\delta_t=\|x^t-y^t\|$.
By definition of $t_0$ it is clear that $\delta_1=\ldots=\delta_{t_0}=0$. For $t=t_0+1$, we have that $\delta_{t_0+1}=\|\eta_{t_0}(\nabla f_{t_0}(x^{t_0})-\nabla f_{t_0}^{\prime}(y^{t_0})\|\leq 2L\eta_{t_0}$.

Now, we derive a recurrence for $(\delta_t)_{t\in[T]}$:
\begin{align*}
&\delta_{t+1}^2
= \|\proj_{\X}[x^t-\eta_t \nabla f_t(x^t)]-\proj_{\X}[y^t-\eta_t \nabla f_t^{\prime}(y^t)]\|^2
\,\,\leq\,\, \|x^t-y^t-\eta_t( \nabla f_t(x^t)- \nabla f_t^{\prime}(y^t))\|^2 \notag\\
& = \delta_t^2 +\eta_t^2\|\nabla f_t(x^t)-\nabla f_t^{\prime}(y^t)\|^2
-2\eta_t\langle \nabla f_t(x^t)-\nabla f_t^{\prime}(y^t),x^t-y^t\rangle\\
& \leq \delta_t^2 +\eta_t^2\|\nabla f_t(x^t)-\nabla f_t^{\prime}(y^t)\|^2
-2\eta_t\langle\nabla f_t(x^t)-\nabla f_t^{\prime}(x^t),x^t-y^t\rangle
-2\eta_t\langle \nabla f_t^{\prime}(x^t)-\nabla f_t^{\prime}(y^t),x^t-y^t\rangle \\
& \leq \delta_t^2 +\eta_t^2\|\nabla f_t(x^t)-\nabla f_t^{\prime}(y^t)\|^2
+2\eta_t\|\nabla f_t(x^t)-\nabla f_t^{\prime}(x^t)\|\delta_t
-2\eta_t\langle \nabla f_t^{\prime}(x^t)-\nabla f_t^{\prime}(y^t),x^t-y^t\rangle \\
& \leq \delta_t^2+4L^2\eta_t^2+2\eta_ta_t\delta_t,
\end{align*}
where at the last step we use the monotonicity of the subgradient. Note that $$\delta_{t_0+1}\leq \eta_{t_0} \|\nabla f_{t_0}(x^{t_0})-\nabla f_{t_0}^{\prime}(x^{t_0})\| \leq 2L\eta_{t_0}.$$
Hence,
\begin{eqnarray}
\delta_t^2 &\leq&  \textstyle \delta_{t_0+1}^2 + 4L^2\sum_{s=t_0+1}^{t-1}\eta_s^2+2\sum_{s=t_0+1}^{t-1}\eta_sa_s\delta_s
\notag\\
&\leq& \textstyle  4L^2\sum_{s=t_0}^{t-1}\eta_s^2+2\sum_{s=t_0+1}^{t-1}\eta_sa_s\delta_s.
\label{eqn:rec_stab}
\end{eqnarray}
Now we prove the following bound by induction (notice this claim proves the result):
$$ \textstyle \delta_t \leq 2L\sqrt{\sum_{s=t_0}^{t-1}\eta_s^2}+2\sum_{s=t_0+1}^{t-1}\eta_sa_s\delta_s \qquad(\forall t\in[T]). $$
Indeed, the claim is clearly true for $t=t_0$. For the inductive step, we assume it holds for some $t\in[T-1]$. To prove the result we consider two cases: first, when $\delta_{t+1}\leq \max_{s\in[t]}\delta_s$, by induction hypothesis we have
$$ \textstyle \delta_{t+1}\leq \delta_t\leq 2L\sqrt{\sum_{s=t_0}^{t-1}\eta_s^2}+2\sum_{s=t_0+1}^{t-1}\eta_sa_s
\leq 2L\sqrt{\sum_{s=t_0}^{t}\eta_s^2}+2\sum_{s=t_0+1}^{t}\eta_sa_s. $$
In the other case, $\delta_{t+1}> \max_{s\in[t]}\delta_s$, we use \eqref{eqn:rec_stab}
$$
\textstyle \delta_{t+1}^2 \,\leq\,  4L^2\sum_{s=t_0}^{t}\eta_t^2 + 2\sum_{s=t_0+1}^t\eta_sa_s\delta_s
\,\,\leq \,\, 4L^2\sum_{s=t_0}^{t}\eta_t^2 + 2\delta_{t+1}\sum_{s=t_0+1}^t\eta_sa_s,
$$
which is equivalent to
$$ \textstyle \Big(\delta_{t+1}-\sum_{s=t_0+1}^ta_s\eta_s\Big)^2 \,\,\leq\,\,  4L^2\sum_{s=t_0}^{t}\eta_t^2 +\Big(\sum_{s=t_0+1}^t\eta_s a_s\Big)^2.
$$
Taking square root at this inequality, and using the subadditivity of the square root, we obtain the inductive step, and therefore the result.
\end{proof}


\iffull
\subsection{Upper Bounds for the Full Batch GD}


\begin{algorithm}[h!]
	\caption{$\A_{\sf GD}$: Full-batch Gradient Descent}
	\begin{algorithmic}[1]
		\REQUIRE Dataset: $S=(z_1, \ldots, z_n)\in \Z^n$, 
		\# iterations $T$, ~step sizes $\{\eta_t: t\in [T]\}$
		\STATE Choose arbitrary initial point $x^1 \in \X$
		\FOR{$t=1$ to $T-1$\,}
		\STATE $x^{t+1} := \proj_{\X}\left(x^{t}-\eta_t\cdot \nabla F_S(x^{t})\right),$
		\ENDFOR
		\RETURN $\barx^T=\frac{1}{\sum_{t\in[T]}\eta_t}\sum_{t\in[T]}\eta_t x^t$
	\end{algorithmic}
	\label{alg:batchGD}
\end{algorithm}

As a direct corollary of Lemma~\ref{lem:main_lem}, we derive the following upper bound on UAS for the batch gradient descent algorithm.

\begin{thm} \label{thm:UB_GD}
Let $\X\subseteq\B(0,R)$ and $\F=\F_{\X}^0(L)$. The full-batch gradient descent (Algorithm \ref{alg:batchGD}) has \stabname~
$$\sup\limits_{S\simeq S^{\prime}}\delta_{\A_{\sf GD}}(S,S^{\prime})\leq \min\left\{2R,~
4L\,\Big(\frac1n\sum_{t=1}^{T-1}\eta_t+\sqrt{\sum_{t=1}^{T-1}\eta_t^2}\Big)\right\}.$$
\end{thm}



\begin{proof}
The bound of $2R$ is obtained directly from the diameter bound on $\X$. Therefore, we focus exclusively on the second term. Let $S\simeq S^{\prime}$ be arbitrary neighboring datasets, $x^1=y^1$, and consider the trajectories $(x^t)_t,(y^t)_t$ associated with the batch GD method on datasets $S$ and $S^{\prime}$, respectively. We use Lemma~\ref{lem:main_lem} with $f_t=F_{S}$ and $f_t^{\prime}=F_{S^{\prime}}$, for all $t\in[T-1]$. Notice that 
$$ \sup_{x\in{\cal X}}\|\nabla F_S(x)- \nabla F_{S^{\prime}}(x)\|\leq 2L/n,$$
since $S\simeq S^{\prime}$; in particular, $\|\nabla f_t(x^t)-\nabla f_t^{\prime}(x^t)\|\leq a_t$, with $a_t=2L/n$. 
We conclude by Lemma~\ref{lem:main_lem} that for all $t\in[T]$
$$ \|x^t-y^t\| \leq 2L\sqrt{ \sum_{s=1}^{t-1}\eta_s^2 } +\frac{4L}{n}\sum_{s=2}^{t-1} \eta_s. $$
Hence, the stability bound holds for all the iterates, and thus for $\overline{x}^T$ by the triangle inequality.
\end{proof}

\fi

\iffull
\subsection{Upper Bounds for SGD}

Next, we state and prove upper bounds on UAS for two variants of stochastic gradient descent: sampling-with-replacement SGD (Section~\ref{sec:w-replace-sgd}) and fixed-permutation SGD (Section~\ref{sec:perm-sgd}). Here, we give strong upper bounds that hold with high probability (for sampling-with-replacement SGD) and with probability 1 (for fixed-permutation SGD). In Appendix~\ref{app:T_less_n}, we derive tighter upper bounds for these two variants of SGD in the case where the number of iterations $T < $ the number of samples in the data set $n$; however, the bounds derived in this case hold only in expectation.

\subsubsection{Sampling-with-replacement SGD} \label{sec:w-replace-sgd}

Next, we study the uniform argument stability of the sampling-with-replacement stochastic gradient descent (Algorithm~\ref{alg:repl_SGD}). This algorithm has the benefit that each iteration is
extremely cheap compared to Algorithm \ref{alg:batchGD}. 
Despite these savings, we will show that same bound on \UHS holds with high probability.

\fi

\begin{algorithm}[h!]
	\caption{$\A_{\sf rSGD}$: Sampling with replacement SGD}
	\begin{algorithmic}[1]
		\REQUIRE Dataset: $S=(z_1, \ldots, z_n)\in \Z^n$, 
		\# iterations $T$, ~stepsizes $\{\eta_t: t\in [T]\}$
		\STATE Choose arbitrary initial point $x^1 \in \X$
		\FOR{$t=1$ to $T-1$\,}
		\STATE Sample $\bI_t\sim\mbox{Unif}([n])$
		\STATE $x^{t+1} := \proj_{\X}\left(x^{t}-\eta_t\cdot \nabla f(x^t,z_{\bI_t})\right)$
		\ENDFOR
		\RETURN $\barx^T=\frac{1}{\sum_{t\in[T]}\eta_t}\sum_{t\in[T]}\eta_t x^t$ 
	\end{algorithmic}
\label{alg:repl_SGD}
\end{algorithm}

\iffull
We now state and prove our upper bound for sampling-with-replacement SGD.
\begin{thm} \label{thm:UB_repl_SGD}
Let $\X\subseteq\B(0,R)$ and $\F=\F_{\X}^0(L)$. The uniform argument stability of the sampling-with-replacement SGD (Algorithm \ref{alg:repl_SGD}) satisfies:
$$\sup\limits_{S\simeq S'}\ex{\A_{\sf rSGD}}{\delta_{\A_{\sf rSGD}}(S,S^{\prime})}\leq \min\left(2R,~ 4L\left(\sqrt{\sum_{t=1}^{T-1}\eta_t^2}  +\frac{1}{n}\sum_{t=1}^{T-1}\eta_t\right)\right).$$

\noindent Moreover, if $\eta_{t}=\eta >0 ~\forall t$ then, for any pair $(S, S^{\prime})$ of neighboring datasets, with probability at least $1-\exp\left(-n/2\right)$ (over the algorithm's internal randomness), the UAS random variable is bounded as
$$\delta_{\A_{\sf rSGD}}(S,S^{\prime}) \leq \min\left(2R,~ 4L\left(\eta \,\sqrt{T-1}  +\eta\, \frac{T-1}{n}\right)\right).$$
\end{thm}
\fi

\iffull
\begin{proof}
The bound of $2R$ trivially follows from the diameter bound on $\X$. We thus focus on the second term of the bound. Let $S\simeq S^{\prime}$ be arbitrary neighboring datasets, $x^0=y^0$, and consider the trajectories $(x^t)_{t\in[T]},(y^t)_{t\in[T]}$ associated with the sampled-with-replacement stochastic subgradient method on datasets $S$ and $S^{\prime}$, respectively. We use Lemma~\ref{lem:main_lem} with $f_t(\cdot)=f(\cdot,\bz_{\bI_t})$ and $f_t^{\prime}(\cdot)=f(\cdot,\bz_{{\bI_t}^{\prime}})$.
Let us define $\br_t\triangleq \mathbf{1}_{\{\bz_{\bI_t}\neq\bz_{\bI_t}^{\prime}\}}$. Note that at every step $t$, $\br_t=1$ with probability $1-1/n$, and $\br_t=0$ otherwise. Moreover, note that $\{\br_t: ~t\in [T]\}$ is an independent sequence of Bernoulli random variables. Finally, note that $\|\nabla f_t(x^t)-\nabla f_t^{\prime}(x^t)\|\leq 2L \br_t$.


Hence, by Lemma~\ref{lem:main_lem}, for any realization of the trajectories of the SGD method, we have
\begin{align}
  \forall t\in [T]:\quad
  \|x^t-y^t\| &\leq 2L\sqrt{ \sum_{s=1}^{t-1}\eta_s^2 } +4L\sum_{s=1}^{t-1} \br_s\eta_s \leq \Delta_T,\label{bound:unif_t}
\end{align}
where $\Delta_T\triangleq 2L\sqrt{ \sum_{s=1}^{T-1}\eta_s^2 } +4L\sum_{s=1}^{T-1} \br_s\eta_s$. Taking expectation of (\ref{bound:unif_t}), we have
$$ \forall t\in [T]:\quad \EE\left[\|x^t-y^t\|\right] \leq \EE\left[\Delta_T\right]= 2L\sqrt{ \sum_{s=1}^{T-1}\eta_s^2 } +\frac{4L}{n}\sum_{s=1}^{T-1}\eta_s. $$
This establishes the upper bound on UAS but only in expectation. Now, we proceed to prove the high-probability bound. Here, we assume that the step size is fixed; that is, $\eta_t=\eta >0$ for all $t\in [T-1]$. Note that each $\br_s, s\in [T],$ has variance $\frac{1}{n}\left(1-\frac{1}{n}\right)<\frac{1}{n}$. Hence, by Chernoff's bound\footnote{Here, we are applying a bound for (scaled) Bernoulli rvs where the exponent is expressed in terms of the variance.}, we have
\begin{align*}
\PP\left[\eta \sum_{s=1}^{T-1} \br_s \geq \eta \frac{T-1}{n}+\eta \sqrt{T-1}
\right] &\leq \exp\left(-\,\frac{\eta^2 (T-1)}{2\eta^2\,\frac{T-1}{n}}\right) = \exp\left(-\frac{n}{2}\right).
\end{align*}
Therefore, with probability at least $1-\exp\left(-n/2\right)$, we have
$$\Delta_T\leq 3L\eta\sqrt{ T-1 } +\frac{4L}{n}\eta (T-1).$$
Putting this together with (\ref{bound:unif_t}), with probability at least $1-\exp\left(-n/2\right)$, we have
$$\forall t\in [T]:\quad \|x^t-y^t\| \leq 3L\eta\sqrt{ T-1 } +\frac{4L}{n}\eta (T-1). $$
Finally, by the triangle inequality, we get that with probability at least $1-\exp\left(-n/2\right)$, the same stability bound holds for the average of the iterates $\overline{x}^T$, $\overline{y}^T$.
\end{proof}
\fi


\iffull

\subsubsection{Upper Bounds for the Fixed Permutation SGD} \label{sec:perm-sgd}

In Algorithm \ref{alg:perm_SGD}, we describe the fixed-permutation stochastic gradient descent. This algorithm works in epochs, where each epoch is a single pass on the data. The order in which data is used is the same across epochs, and is given by a permutation $\pi$. 
The algorithm can be alternatively described without the epoch loop simply by
\begin{equation}
\label{eqn:single_iter_pr_SGD}
x^{t+1} = \proj_{\X}\left(x^{t}-\eta_t\cdot \nabla f(x^{t}, z_{\boldsymbol{\pi}(t\, \mbox{\footnotesize mod } n)})\right) \qquad (\forall t\in[nK]).
\ifthenelse{\boolean{fullver}}{}{\vspace{-0.5cm}}
\end{equation}
\ifthenelse{\boolean{fullver}}{We will use this description for stability analysis, since it is more convenient.}{}


\begin{algorithm}
	\caption{$\A_{\sf PerSGD}$: Fixed Permutation SGD} \label{alg:perm_SGD}
	\begin{algorithmic}[1]
		\REQUIRE Dataset $S=(z_1, \ldots, z_n)\in \Z^n$, 
		\# rounds $K$, total \# steps $T\triangleq nK$, step sizes, $\{\eta_t\}_{t\in [nK]}$ 
		${\pi}:[n]\rightarrow [n]$ permutation over $[n]$
		\STATE Choose arbitrary initial point $x_{n+1}^0 \in \X$
		\FOR{$k=1,\ldots,K$}
		\STATE $x_1^k=x_{n+1}^{k-1}$
		\FOR{$t=1$ to $n$\,}
		\STATE $x^k_{t+1} := \proj_{\X}\left(x^k_{t}-\eta_{(k-1)n+t}\cdot \nabla f(x^k_{t}, z_{\pi(t)})\right)$
		\ENDFOR
		\STATE $\overline{\eta}_k=\sum_{t=1}^n\eta_{(k-1)n+t} $
		\ENDFOR
		\RETURN $\barx^K=\frac{1}{\sum_{k\in[K]}\overline{\eta}_k} \sum_{k\in[K]}\overline{\eta}_{k} \cdot x_{1}^k$
	\end{algorithmic}
\end{algorithm}

We show that the same UAS bound of batch gradient descent and sampling-with-replacement SGD holds for the fixed-permutation SGD. We also observe that a slightly tighter bound can be achieved if we consider \textit{the expectation guarantee} on UAS when $\boldsymbol{\pi}$ is chosen uniformly at random. We leave these details to Theorem \ref{thm:UB_perm_SGD_2} in the Appendix. 

In the next result, we assume that the sequence of step sizes $\left(\eta_t\right)_{t\in [T]}$ is non-increasing, which is indeed the case for almost all known variants of SGD.

\begin{thm} \label{thm:UB_perm_SGD}
Let $\X\subseteq\B(0,R)$, $\F=\F_{\X}^0(L)$, and $\pi$ be any permutation over $[n]$. Suppose the step sizes $\left(\eta_t\right)_{t\in [T]}$ form a non-increasing sequence. Then the \stabname of the fixed-permutation SGD (Algorithm \ref{alg:perm_SGD}) is bounded as
$$\sup\limits_{S\simeq S^{\prime}}\delta_{\A_{\sf PerSGD}}(S,S^{\prime})\leq \min\left\{2R,~2L \left(\sqrt{\sum_{t=1}^{T-1}\eta_t^2} +\frac{2}{n} \sum_{t=1}^{T-1}\eta_t\right)\right\}.$$
\end{thm}

\begin{proof}
Again, the bound of $2R$ is trivial. Now, we show the second term of the bound. Let $S\simeq S^{\prime}$ be arbitrary neighboring datasets, $x^1=y^1$, and consider the trajectories $(x^t)_{t\in[T]},(y^t)_{t\in[T]}$ associated with the fixed permutation stochastic subgradient method on datasets $S$ and $S^{\prime}$, respectively. 
Since the datasets $S\simeq S^{\prime}$ are arbitrary, we may assume without loss of generality that $\pi$ is the identity, whereas the perturbed coordinate $\bI=i$ is arbitrary. We use Lemma~\ref{lem:main_lem} with $f_t(\cdot)=f\left(\cdot,\bz_{\left(t ~\mbox{\footnotesize mod }n\right)}\right)$ and
$f_t^{\prime}(\cdot)=f\left(\cdot,\bz^{\prime}_{\left(t ~\mbox{\footnotesize mod }n\right)}\right)$.  
It is easy to see then that $\|\nabla f_t(x^t)-\nabla f_t^{\prime}(x^t)\|\leq a_t$, with
$a_t=2L \cdot\mathbf{1}_{\{(t ~\mbox{\footnotesize mod }n)=i\}}$, where $\mathbf{1}_{\{\mathsf{condition}\}}$ is the indicator of $\mathsf{condition}$. Hence, by Lemma~\ref{lem:main_lem}, we have
\begin{eqnarray*}
\|x^t-y^t\|
&\leq& 2L\sqrt{\sum_{s=1}^{t-1}\eta_s^2} +4L\sum_{r=1}^{\lfloor (t-1)/n \rfloor}\eta_{rn+i} \\
&\leq& 2L\sqrt{\sum_{s=1}^{t-1}\eta_s^2} +\frac{4L}{n} \sum_{r=1}^{t-1}\eta_s,
\end{eqnarray*}
where at the last step we used the fact that $(\eta_t)_{t\in[T]}$ is non-increasing; namely, for any $r\geq 1$
$$ \eta_{rn+i} \leq \frac1n \sum_{s=(r-1)n+i+1}^{rn+i} \eta_s. $$
Since the bound holds for all the iterates, using triangle inequality, it holds for the output $\barx^K$ averaged over the iterates from the $T/n$ epochs.
\end{proof}

\subsection{Discussion of the upper bounds: examples of specific instantiations}



The upper bounds on stability from this section all behave very similarly. Let us explore the consequences of the obtained rates in terms of generalization bounds for different choices of the step size sequence. As a case study, we will consider excess risk bounds for the full-batch subgradient method (Algorithm \ref{alg:batchGD}), but similar conclusions hold for all the variants that we studied. We emphasize that prior to this work, no dimension-independent bounds on the excess risk were known of this method (specifically, for nonsmooth losses and without explicit regularization).


To bound the excess risk, we will use the risk decomposition, eqn.~\eqref{eqn:risk_decomp}.  For simplicity, we will only be studying excess risk bounds in expectation (in Section \ref{sec:Applications} we consider stronger, high probability, bounds). In this case, the stability implies generalization result (Theorem \ref{thm:stab-to-gen}) simplifies to \cite{BousquettE02,HardtRS16}
$$ \EE_{\bS}[F_{\bS}({\cal A}(\bS))-F_{\cal D}({\cal A}(\bS))] \leq \sup_{S\simeq S^{\prime}}\delta_{{\cal A}}(S,S^{\prime}). $$
Finally, the approximation error (Lemma \ref{lem:approx-err-bd}) simplifies as well: it is upper bounded by 0 in expectation.

\begin{itemize}
    \item {\em Fixed stepsize:} Let $\eta_t\equiv\eta>0$. By Thm.~\ref{thm:UB_GD}, UAS is bounded by $4L\sqrt{T}\eta+\frac{4LT\eta}{n}$. On the other hand, the standard analysis of subgradient descent guarantees that $\varepsilon_{\mbox{\footnotesize opt}}(\A_{\sf GD})\leq \frac{R^2}{2\eta T}+\frac{\eta L^2}{2}$. Therefore, by the expected risk decomposition \eqref{eqn:risk_decomp}
    \begin{eqnarray*}\EE_{\bS}[\varepsilon_{\mbox{\footnotesize{risk}}}(\A_{\sf GD})] &\leq& \EE_{\bS}[\generr(\A_{\sf GD})] +\EE_{\bS}[\opterr(\A_{\sf GD})] \leq 4L^2\sqrt{T}\eta+\frac{4L^2T\eta}{n}+\frac{R^2}{2\eta T}+\frac{\eta L^2}{2}.
    \end{eqnarray*}
    If we consider the standard method choice, $\eta=R/[L\sqrt n]$ and $T=n$, the bound above is at least $4LR$ (due to the first term). Consequently, upper bounds obtained from this approach are vacuous.

    In order to deal with the $L\sqrt T \eta$ term, we need to substantially moderate our stepsize, together with running the algorithm for longer. For example, $\eta=\frac{R}{4L\sqrt{Tn}}$ gives
    $\EE_{\bS}[\varepsilon_{\mbox{\footnotesize{risk}}}(\A_{\sf GD})]\leq \frac{2LR}{\sqrt n}+\frac{2LR\sqrt n}{\sqrt T}+\frac{R\sqrt T}{n^{3/2}}$, so by choosing $T=n^2$ we obtain an expected excess risk bound of $O(LR/\sqrt n)$, which is optimal. We will see next that it is not possible to obtain the same rates from this bound if $T=o(n^2)$, for any choice of $\eta>0$. It is also an easy observation that, at least for constant stepsize, it is not possible to recover the optimal excess risk if $T=\omega(n^2)$.

    \item {\em Varying stepsize:} For a general sequence of stepsizes the optimization guarantees of Algorithm \ref{alg:batchGD} are the following
    $$\EE_{\bS}[\opterr(\A_{\sf GD})] \leq \frac{R^2}{2\sum_{t=1}^{T-1}\eta_t}+\frac{L^2\sum_{t=1}^{T-1}\eta_t^2}{2}. $$
    From the risk decomposition, we have
    \begin{eqnarray*}\EE_{\bS}[\varepsilon_{\mbox{\footnotesize{risk}}}(\A_{\sf GD})] &\leq& \EE_{\bS}[\generr(\A_{\sf GD})] +\EE_{\bS}[\opterr(\A_{\sf GD})] \\
    &\leq& 4L^2\sqrt{\sum_{t=1}^{T-1}\eta_t^2}+\frac{4L^2}{n}\sum_{t=1}^{T-1}\eta_t+\frac{R^2}{2\sum_{t=1}^{T-1}\eta_t}+\frac{L^2\sum_{t=1}^{T-1}\eta_t^2}{2}.
    \end{eqnarray*}
    In fact, we can show that any choice of step sizes that makes the quantity above $O(LR/\sqrt n)$ must necessarily have $T=\Omega(n^2)$. Indeed, notice that in such case
    \begin{eqnarray*} \frac{R^2}{2\sum_{t=1}^{T-1}\eta_t} = O\Big(\frac{LR}{\sqrt n}\Big); & 4L^2\sqrt{\sum_{t=1}^{T-1}\eta_t^2}=O\Big(\frac{LR}{\sqrt n}\Big) \\
    \Longleftrightarrow \qquad
    \sum_{t=1}^{T-1}\eta_t = \Omega\Big(\frac{R\sqrt n}{L}\Big); & \sqrt{\sum_{t=1}^{T-1}\eta_t^2}=O\Big(\frac{R}{L\sqrt n} \Big).
    \end{eqnarray*}
    Therefore, by Cauchy-Schwarz inequality,
    $$ \Omega\Big(\frac{R\sqrt n}{L}\Big)=\sum_t \eta_t \leq \sqrt{T}\sqrt{\sum_t \eta_t^2} =O\Big(\frac{R\sqrt{T}}{L\sqrt n} \Big) \quad\Longrightarrow\quad T=\Omega(n^2). $$
\end{itemize}
The high iteration complexity required to obtain optimal bounds motivates studying whether it is possible to improve our \stabname bounds. We will show that, unfortunately, they are sharp up to absolute constant factors.


\fi

\section{Lower Bounds on \Stabname}

\label{sec:LowerBounds}

In this section we provide matching lower bounds for the previously studied first-order methods. These lower bounds show that our analyses are tight, up to absolute constant factors.

We note that it is possible to prove a general purpose lower bound on stability by appealing to sample complexity lower bounds for stochastic convex optimization \citep{nemirovsky1983problem}.
This approach in the smooth convex case was first studied in \citep{Chen:2018}; there, these lower bounds are sharp.
However, in the nonsmooth case they are very far from bounds in the previous section. 
The idea is that for sufficiently small step size, 
a first-order method must incur $\Omega(LT\eta/n)$ uniform stability.
\ifthenelse{\boolean{fullver}}{
\begin{obs} \label{obs:base_lb}

Let ${\cal A}$ be a $\gamma$-uniformly stable stochastic convex optimization algorithm  with  $\gamma=s(T)/n$, where $s(T)$ is increasing and $\lim_{T\rightarrow +\infty}s(T)=+\infty$. By the lower bound on the optimal risk of nonsmooth convex optimization, $\epsrisk\geq \frac{LR}{C_1 \sqrt{n}}$, where $C_1>0$ is a universal constant \citep{nemirovsky1983problem}. This, combined with the risk decomposition \eqref{eqn:risk_decomp}, implies that
$$\epsopt\geq \frac{LR}{C_1\sqrt n}-\frac{s(T)}{n}
= -s(T)\Big( \frac{1}{\sqrt n}-\frac{LR}{2C_1 s(T)} \Big)^2+\frac{(LR)^2}{4C_1^2 s(T)}.$$
By our assumption on $s(T)$, for $T$ sufficiently large, there always exists $n$ such that
$$ \frac{4C_1s(T)}{3LR} \leq \sqrt{n} \leq  \frac{4C_1s(T)}{LR}  $$
which leads to $\epsopt\geq \frac{(LR)^2}{C_2 s(T)}$, where $C_2>0$ is a universal constant.

If algorithm ${\cal A}$ is based on $T$ subgradient iterations with constant step size $\eta>0$ (these could be either stochastic, batch or minibatch), by standard analysis, the optimization guarantee of such algorithm is $\epsopt\leq\frac12(\frac{R^2}{\eta T}+\eta L^2)$. Both bounds in combination give
$$ s(T) \geq \frac{2(LR)^2}{C_2 (\eta L^2+R^2/[\eta T])}
=\frac{2(LR)^2\eta T}{C_2(\eta^2TL^2+R^2)}. $$
If we further assume that $\eta \leq (R/L)/\sqrt{T}$ (notice $\eta= (R/L)/\sqrt{T}$ minimizes the optimization error), then $s(T)\geq L^2\eta T/C_2$. We also emphasize all the choices of step size that we will make to control generalization error will lie in this range. 
\end{obs}
}{
Details of this lower bound can be found on Appendix \ref{app:gen_LB_FOM}.
}
This reasoning leads to an
$\Omega(LT\eta/n)$ lower bound on \stabname, that can be added to any other lower bound we can prove on specific algorithms that enjoy rates as of gradient descent.

Next we will prove finer lower bounds on the \UHS of specific algorithms.
For this, note that the objective functions we use are polyhedral, thus the subdifferential is a polytope at any point. Since the algorithm should work for
any oracle, we will let the subgradients provided to be extreme points, $\nabla f(x,z)\in\mbox{ext}(\partial f(x,z))$. Moreover, we can make adversarial choices of the chosen subgradient.

\subsection{Lower Bounds for Full Batch GD}

\begin{thm} \label{thm:LB_GD_basic}
Let $\X={\cal B}(0,1)$, ${\cal F}={\cal F}_{\X}^0(1)$ and $d\geq \min\{T,1/\eta^2\}$. For the full-batch gradient descent (Alg.~\ref{alg:batchGD}) with constant step size $\eta>0$, there exist $S\simeq S'$ such that the \UHS is lower bounded as
$\delta_{\A_{\sf GD}(S, S')}=\Omega(\min\{1,\eta\sqrt{T}+\eta T/n\}).$ 
\end{thm}
\ifthenelse{\boolean{fullver}}{
\begin{proof}
Let $D\triangleq\min\{T,1/\eta^2\}\leq d$, and 
$\nu,K>0$. We consider $\Z= \{0,1\}$, and the objective
function
$$ f(x,z) =
\left\{ 
\begin{array}{ll}
\max\{0,x_1-\nu,\ldots,x_{D}-\nu\}           & \mbox{ if } z=0 \\
\langle r,x\rangle/K& \mbox{ if } z=1, 
\end{array}
\right.
$$
where $r=(-1,\ldots,-1,0,\ldots,0)$ (i.e., supported on the first $D$ coordinates).
Notice that for normalization purposes, we need $K\geq \sqrt{D}$; furthermore, we will choose $K$ sufficiently large such that $T\sqrt{D}/[nK]=o(1)$.
Consider the data sets 
$S \simeq S^{\prime}$, leading to the empirical objectives:
$$ F_{S}(x)= \frac{1}{nK}\langle r,x\rangle 
+\frac{n-1}{n}\max\{0,x_1-\nu,\ldots,x_{D}-\nu\} \,\,\mbox{and} \,\,
F_{S^{\prime}}(x)= \max\{0,x_1-\nu,\ldots,x_{D}-\nu\}.$$

Let $(x^t)_{t\in[T]}$ and $(y^t)_{t\in[T]}$ be the trajectories of the algorithm over datasets $S$ and $S^{\prime}$, respectively, initialized from $x^1=y^1=0$. Clearly, $y^t=0$ for all $t$. Now 
$x^2=-\frac{\eta}{nK}\,r$; choosing $\nu<\eta/(nK)$, we have $\nabla f(x^2,z)=-\frac{\eta}{nK}\,r+\frac{n-1}{n} e_1$, and hence
$ x^3= -\frac{2\eta }{nK}\,r-\eta\frac{n-1}{n} e_1.$ 
Sequentially, the method will perform cumulative subgradient steps on $e_2,e_3\ldots,e_{D}$. 
In particular, for any $t\in [D+1],$ we have $x^{t+1}=-t\frac{\eta}{nK}\,r-\eta\frac{n-1}{n} \sum_{s=1}^{t-1}e_{s}.$ 

By orthogonality of the subgradients and given our choice of $K$, we conclude that 
\begin{eqnarray*}
\|x^{D+2}-y^{D+2}\|
&=&\|x^{D+2}\|\geq \frac{\eta}{2}\Big\| \sum_{t=1}^{D} e_t \Big\|-\eta\frac{D\sqrt{D}}{nK}\\
&\geq& \frac{\eta}{2}\sqrt{D}-\eta\cdot o(1)=\Omega(\eta \sqrt{D}) \\
&=&\Omega(\min\{1,\eta\sqrt T\}),
\end{eqnarray*}
and further subgradient steps $t=D+1,\ldots,T$ are only given by the linear term, $r/[nK]$, which are negligible perturbations.


We finish by arguing that averaging does not help. First, in the case $D=T$:
\begin{eqnarray*}
\delta(\A_{\sf GD}) &=&
\|\barx^T\|\geq \frac{\eta}{2} \Big\|\frac1T \sum_{t=1}^{ T}\sum_{s=1}^{t-2} e_s\Big\| -o(1)
= \frac{\eta}{2} \Big\| \sum_{s=1}^T\frac{T-s-2}{T} e_s\Big\|-o(1)\\
&\geq& \frac{\eta}{4} \Big\| \sum_{s\leq T/2-2} e_s\Big\|-o(1) =
\Omega(\eta\sqrt{T}). 
\end{eqnarray*}
And second, in the case $D=1/\eta^2$:
\begin{eqnarray*}
\delta(\A_{\sf GD}) &=& \|\overline{x}^T\| 
\geq \frac{\eta}{2T}\Big\| \sum_{t=1}^{D+2} \sum_{s=1}^{t-2}e_s+\sum_{t=D+3}^T\sum_{s=1}^{D} e_s \Big\|-o(1) \\
&=& \frac{\eta}{2T}\Big\| \sum_{s=1}^{D}(D-s-1)e_s+\sum_{s=1}^{D}(T-D+2)e_s \Big\|-o(1)\\
&=&  \frac{\eta}{2}\Big\| \sum_{s=1}^{D-1} \frac{T-s+1}{T} e_s \Big\|-o(1)
\geq \frac{\eta}{4}\Big\| \sum_{s=1}^{D/2}  e_s \Big\|-o(1)=
\Omega(\sqrt{D}\eta)=\Omega(1).
\end{eqnarray*}
Finally, the additional term $\Omega(\eta T/n)$ in the lower bound is obtained by \ifthenelse{\boolean{fullver}}{Observation \ref{obs:base_lb}}{Appendix \ref{app:gen_LB_FOM}}.
\end{proof}
}{
The proof of this result is deferred to Appendix \ref{app:LB_GD_basic}, due to space considerations.
}

\subsection{Lower Bounds for SGD Sampled with Replacement}

We use a similar construction as from the previous result to prove a sharp lower bound on the
\stabname for stochastic gradient descent where the sampling is with replacement. 

\begin{thm} \label{thm:LB_rSGD_basic}
Let $\X={\cal B}(0,1)$, ${\cal F}={\cal F}_{\X}^0(1)$, and $d\geq \min\{T,1/\eta^2\}$. For the sampled with replacement stochastic gradient descent (Algorithm \ref{alg:repl_SGD}) with constant step size $\eta>0$, there exist $S\simeq S^{\prime}$ such that the \stabname satisfies 
$\EE[\delta_{\A_{\sf rSGD}}(S,S^{\prime})]=
\Omega\Big(\min\Big\{1,\frac{T}{n}\Big\}\eta\sqrt T +\frac{\eta T}{n}\Big).$
\end{thm}

\begin{proof}
Let $D\triangleq\min\{T,1/\eta^2\}\leq d$, and $\nu>0$, $K\geq \sqrt D$. Consider ${\cal Z}=\{0,1\}$ and define
$$
f(x,z) =
\left\{
\begin{array}{ll}
\max\{0,x_1-\nu,\ldots,x_D-\nu\}           & \mbox{ if } z=0 \\
\langle r,x\rangle/K& \mbox{ if } z=1,
\end{array}
\right.
$$
where $r=(-1,\ldots,-1,0,\ldots,0)$ (i.e., supported on the first $D$ coordinates).
Let the random sequence of indices used by the algorithm: $(\bI_t)_{t\geq 0} \stackrel{i.i.d.}{\sim}\mbox{Unif}([n])$. 
Let $S=(1,0,\ldots,0)$ and $S^{\prime}=(0,0,\ldots,0)$ be neighboring datasets, and denote by $(x^t)_t$ and $(y^t)_t$ the respective stochastic gradient descent trajectories on $S$ and $S^{\prime}$, initialized at $x^1=y^1=0$. It is easy to see that under $S^{\prime}$, we have $y^t=0$ for all $t\in[T]$. Now, suppose that $\nu<\eta/K$. Then, we only have $x^t=0$ for all $t\leq \tau$, where $\tau:=\inf\{t\geq 1:\,\bI_t=1 \}$. After time $\tau$, $x^{\tau+1}=-\eta r/K$, and consequently
$x^{\tau+1+j}=-\frac{\eta \bk(\tau+j)}{K} r - \eta \sum_{s=1}^{j-\bk(\tau+j)+1} e_{s},$ for all $j \in [D+\mathbf{k}(\tau+j)-1]$, where $\mathbf{k}(t)\triangleq |\{s\in[t]:\bI_s=1\}|$. Note that conditioned on any fixed value for $\tau,$ ~$\bk(\tau+j)\leq j+1$.

Let $\delta_T=\|x^T-y^T\|=\|x^T\|$. Hence, we have $\delta_T\geq\eta\|\sum_{s=1}^{T-\tau-\mathbf{k}(T-1)}e_s\|-\eta \mathbf{k}(T-1)\sqrt{D}/K\geq \eta\sqrt{T-\mathbf{k}(T-1)-\tau}-\eta T \sqrt D /K$.
Let $\mathbf{k}=\mathbf{k}(T-1)$ from now on. 
Note that conditioned on any value for $\tau,$ ~$\bk-1$ is a binomial random variable taking values in $\{0, \ldots, T-1-\tau\}$. Hence, conditioned on $\tau=t$, by the binomial tail, we always have  $\PP[\mathbf{k}>T/2~|~\tau=t]\leq\exp(-T/4)$ for all $t\in [T]$ (in particular, this conditional probability is zero when $t\geq T/2$). Also, note that the same upper bound is valid without conditioning on $\tau$.
Hence, by the law of total expectation, we have
\begin{eqnarray*}
\EE[\delta_T]
&=& \EE[\delta_T| ~\bk\leq T/2]\cdot\PP[\bk\leq T/2]+\EE[\delta_T | ~\bk> T/2]\cdot\PP[\bk> T/2]
\geq c\,\EE[\delta_T | ~\bk\leq T/2]
\end{eqnarray*}
where $c=(1-\exp(-T/4))=\Omega(1)$. Hence,
\begin{eqnarray*}
\EE[\delta_T]&\geq& c \sum_{t=1}^{T/2} \EE[\delta_T|\tau=t,~ \bk\leq T/2]\,\PP[\tau=t | ~\bk\leq T/2]\\
&\geq& c^2\sum_{t=1}^{T/2} \EE[\delta_T|\tau=t,~ \bk\leq T/2]\,\PP[\tau=t]\\
&\geq& c^2\frac{\eta}{n}\sum_{t=1}^{T/2} \sqrt{T-T/2-t}\big(1-\frac1n\big)^{t-1}-c^2\eta \sqrt D T/K.
\end{eqnarray*}
We choose $K$ sufficiently large such that
$\eta \sqrt D T/K=o(\eta\min\{T^{3/2}/n,\sqrt{T}\})$.
\ifthenelse{\boolean{fullver}}{
If $T\leq n$ then
$$ \EE[\delta_T] \geq c^2\frac{\eta}{n} \sum_{t=1}^{T/2} \sqrt{t} \big(1-\frac1n\big)^{n-2}-c^2\frac{\eta \sqrt D T}{K}
\geq c^2\frac{\eta e^{-1}}{n}\sum_{t=1}^{T/2} \sqrt{t}-o(\frac{\eta T^{3/2}}{n})
= \Omega\Big( \frac{\eta T^{3/2}}{n} \Big)$$
and if $T>n$ then
$$ \EE[\delta_T]
\geq c^2\frac{\eta}{n} \sum_{t=1}^{n/4} \sqrt{T/2-n/4} \, e^{-1} -o(\eta\sqrt{T})
= \Omega( \eta\sqrt{T})$$
}{
Hence, we have
$$
\EE[\delta_T]\geq
\left\{
\begin{array}{ll}
\frac{c^2\eta}{n} \sum_{t=1}^{T/2} \sqrt{t} \big(1-\frac1n\big)^{n-2}-\frac{c^2\eta \sqrt D T}{K}
\geq \frac{c^2\eta e^{-1}}{n}\sum_{t=1}^{T/2} \sqrt{t}-o(\frac{\eta T^{3/2}}{n})
= \Omega\Big( \frac{\eta T^{3/2}}{n} \Big) & \mbox{ if } T\leq n\\
\frac{c^2\eta}{n} \sum_{t=1}^{n/4} \sqrt{T/2-n/4} \, e^{-1} -o(\eta\sqrt{T})
= \Omega( \eta\sqrt{T}) &\mbox{ if }
T>n.
\end{array}
\right.
$$
}
This gives a lower bound on $\EE[\delta_T]$. Proving that $\overline{x}^T$ satisfies the same lower bound is analogous to the proof in Theorem~\ref{thm:LB_GD_basic}. Finally, $\Omega(\eta T/n)$ can be added to the lower bound by \ifthenelse{\boolean{fullver}}{Observation \ref{obs:base_lb}}{Appendix \ref{app:gen_LB_FOM}}.
\end{proof}


\subsection{Lower Bounds for the Fixed Permutation Stochastic Gradient Descent}

Finally, we study fixed permutation SGD. \ifthenelse{\boolean{fullver}}{}{The proof of this result is deferred to Appendix \ref{app:LB_perm_SGD_basic}.}

\begin{thm} \label{thm:LB_perm_SGD_basic}
Let $\X={\cal B}(0,1)$, ${\cal F}={\cal F}_{\X}^0(1)$ and $d\geq \min\{T,1/\eta^2\}$. For the fixed permutation stochastic gradient descent (Algorithm \ref{alg:perm_SGD}) with constant step size $\eta>0$, there exist $S\simeq S^{\prime}$ such that the \stabname is lower bounded by
$\EE[\delta_{\A_{\sf PerSGD}}(S,S^{\prime})]=
\Omega\Big(\min\Big\{1,\frac{T}{n}\Big\}\eta\sqrt T +\frac{\eta T}{n}\Big).$
\end{thm}


\ifthenelse{\boolean{fullver}}{

\begin{proof}
We consider the same function class of Thm.~\ref{thm:LB_GD_basic}, and neighbor datasets $S^{\prime}=(0,0,\ldots,0)$, $S=(1,0,\ldots,0)$. We will assume in what follows that $D=\min\{T,1/\eta^2\}$, $K$ is sufficiently large and $\nu<\eta \|r\|/K$. 
Let $(x^t)_{t\in[T]}$ and $(y^t)_{t\in[T]}$ be the trajectories of Algorithm \ref{alg:perm_SGD} over datasets $S,S^{\prime}$ respectively, both initialized at $x^1=y^1=0$.
Let now $\tau=\boldsymbol{\pi}^{-1}(1)\sim\Unif[n]$. Arguing as in Thm.~\ref{thm:LB_GD_basic}, we have that $y^t=0$ for all $t$, whereas
$$
x^{t+1} = 
\left\{ 
\begin{array}{ll}
0 & t< \tau\\
-\frac{\eta (1+\lfloor t/n\rfloor) r}{K}-\eta\sum_{s=1}^{t-\tau-(1+\lfloor t/n\rfloor)} e_{s} & \tau\leq t\leq \tau+D.
\end{array}
\right.
$$
Later iterations 
will satisfy $\|x^t\|=1-o(1)$ if $D=1/\eta^2$ (and otherwise the algorithm stops earlier). 
Therefore, for all $t\in [T]$,
\begin{eqnarray*}
\EE_{\boldsymbol{\pi}}[\|x^t-y^t\|]
&=&\sum_{s=1}^n\EE[\|x^t-y^t\|\,|\,\tau=s]\PP[\tau=s]\\
&\geq&\frac{\eta}{2n}\sum_{s=1}^{\min\{t,n\}} \sqrt{t-s} - \eta\cdot o(1)\\
&=&
\left\{
\begin{array}{ll}
\Omega(\frac{\eta t^{3/2}}{n}) & \mbox{ if } t\leq n\\
\Omega(\eta\sqrt t) &\mbox{ if } t>n. 
\end{array}
\right.
\end{eqnarray*}
Notice that we used above that $K$ is such that $T\sqrt{D}/nK=o(1)$. 
Analogously as in Thm.~\ref{thm:LB_GD_basic}, we can obtain the same conclusion for $\barx^T$. The lower bound of $\eta T/n$ can be added by \ifthenelse{\boolean{fullver}}{Observation \ref{obs:base_lb}}{Appendix \ref{app:gen_LB_FOM}}, so the result follows.
\end{proof}
}{} 
\section{Generalization Guarantees for Multi-pass SGD}
\label{sec:MultipassSGD}


One important implication of our stability bounds 
is that they provide non-trivial generalization error guarantees for multi-pass SGD on nonsmooth losses. Multi-pass SGD is one of the most extensively used settings of SGD in practice, where SGD is run for $K$ passes (epochs) over the dataset (namely, the number of iterations $T=K n$). To the best of our knowledge, aside from the dimension-dependent bounds based on uniform convergence \citep{ShwartzSSS10}, no generalization error guarantees are known for the multi-pass setting on general nonsmooth convex losses. Given our uniform stability upper bounds, we can prove the following generalization error guarantees for the multi-pass setting of sampling-with-replacement SGD. Analogous results can be obtained for fixed-permutation SGD \iffull. \fi

\begin{thm} \label{thm:gen_bds_replSGD}
Running Algorithm \ref{alg:repl_SGD} for $K$ passes (i.e., for $T=K n$ iterations) with constant stepsize $\eta_t=\eta>0$ yields the following generalization error guarantees:
$$ |\EE_{\A_{\sf rSGD}} [\generr(\A_{\sf rSGD})]| \leq 4L^2 \eta\left( \sqrt{Kn}  + K\right), $$
and there exists $c>0$, such that for any $0<\theta<1$, with probability~$\geq 1-\theta-\exp(-n/2)$,
$$ |\generr(\A_{\sf rSGD})|
\leq c ~\Bigg( L^2 \eta\left(\sqrt{K n}+ K \right) \log(n) \log(n/\theta) +LR\sqrt{\frac{\log(1/\theta)}{n}}~\Bigg). $$
\end{thm}


\begin{proof}

First, by the expectation guarantee on UAS given in Theorem~\iffull\ref{thm:UB_repl_SGD} \fi together with the fact that the losses are $L$-Lipschitz, it follows that Algorithm~\ref{alg:repl_SGD} (when run for $K$ passes with constant stepsize $\eta$) is $\gamma$-uniformly stable, where $\gamma = 4L^2\left(\eta \sqrt{Kn}  +\eta K\right)$. Then, by \cite[Thm.~2.2]{HardtRS16}, we have
$$ |\EE_{\A_{\sf rSGD}}[\generr(\A_{\sf rSGD})]| \leq \gamma. $$
For the high-probability bound, we combine the high-probability guarantee on UAS given in Theorem~\iffull\ref{thm:UB_repl_SGD} \fi with Theorem~\ref{thm:stab-to-gen} to get the claimed bound.
\end{proof}

These bounds on generalization error can be used to obtain excess risk bounds using the standard risk decomposition (see \eqref{eqn:risk_decomp}). In practical scenarios where one can certify small optimization error for multi-pass SGD, Thm.~\ref{thm:gen_bds_replSGD} can be used to readily estimate the excess risk. In Section~6.2 we provide worst-case analysis showing that multi-pass SGD is guaranteed to attain the optimal excess risk of $\approx LR/\sqrt n$ within $n$ passes (with appropriately chosen constant stepsize).


\iffull
\section{Implications of Our Stability Bounds}
\fi
\label{sec:Applications}

\iffull
\subsection{Differentially Private Nonsmooth Stochastic Convex Optimization}
\fi

Now we show an application of our stability upper bound to  \emph{differentially private} stochastic convex optimization (DP-SCO). 
Here, the input sample to the stochastic convex optimization algorithm is a sensitive and private data set, thus the algorithm is required to satisfy the notion of $(\alpha, \beta)$-differential privacy.
A randomized  algorithm $\A$ is $(\alpha,\beta)$-differentially private if, for any
pair of datasets $S\simeq S'$, and for all events $\cO$ in the output range of $\A$, we have
\ifthenelse{\boolean{fullver}}{$$\pr{}{\A(S)\in \cO} \leq e^{\alpha} \cdot \pr{}{\A(S')\in \cO} +\beta ,$$}{
$\pr{}{\A(S)\in \cO} \leq e^{\alpha} \cdot \pr{}{\A(S')\in \cO} +\beta ,$
}
where the probability is taken over the random coins of $\A$ \citep{DMNS06, DKMMN06}. For meaningful privacy guarantees, the typical settings of the privacy parameters are $\alpha<1$ and $\beta \ll 1/n$.



Using our UAS upper bounds, we show that a simple variant of noisy SGD \citep{BST14}, that requires only $n^2$ gradient computations, yields the optimal excess population risk for DP-SCO. In terms of running time, this is a small improvement over the algorithm of \cite{FKT19} for the nonsmooth case, which requires $O(n^2\sqrt{\log 1/\beta})$ gradient computations. More importantly, our algorithm is substantially simpler. For comparison, the algorithm in \citep{FKT19} is based on a multi-phase SGD, where in each phase a separate regularized ERM problem is solved. To ensure privacy, the output of each phase is perturbed with an appropriately chosen amount of noise before being used as the initial point for the next phase.

The description of the algorithm is given in Algorithm~\ref{Alg:NSGD}.
\begin{algorithm}
	\caption{$\A_{\sf NSGD}$: Noisy SGD for convex losses}
	\begin{algorithmic}[1]
		\REQUIRE Private dataset $S=(z_1, \ldots, z_n)\in \Z^n$, 
		step size $\eta$;~ privacy parameters $\alpha \leq 1, ~\beta \ll 1/n$
				\STATE Set noise variance $\sigma^2 := \frac{8\,L^2\,\log(1/\beta)}{\alpha^2}$
				\STATE Choose an arbitrary initial point $x^1 \in \X$
        \FOR{$t=1$ to $n^2-1$\,}
        	\STATE Sample        	$\bI_t\sim\mbox{Unif}([n])$ \label{step:sampling}
        	\STATE $x^{t+1} := \proj_{\X}\left(x^{t}-\eta\cdot\left(\nabla\ell(x^{t}, z_{\bI_t})+\bG_t\right)\right),$ where $\bG_t \sim \N\left(\bzero, \sigma^2 \mathbb{I}_d\right)$ drawn independently each iteration\label{step:grad-step}
            \ENDFOR
            \RETURN $\barx=\frac{1}{n^2}\sum_{t=1}^{n^2} x^t$
	\end{algorithmic}
	\label{Alg:NSGD}
\end{algorithm}

\begin{thm}[Privacy guarantee of $\A_{\sf NSGD}$]\label{thm:priv_Ansgd}
Algorithm~\ref{Alg:NSGD} is $(\alpha, \beta)$-differentially private.
\end{thm}
\iffull
\noindent The proof of the theorem follows the same lines of \cite[Theorem~2.1]{BST14}, but we replace their privacy analysis of the Gaussian mechanism with the
 tighter Moments Accountant method of \cite{abadi2016deep}.
 analysis of \cite{abadi2016deep}.\fi

\begin{thm}[\iffull Excess risk \fi of $\Ansgd$]\label{thm:pop_risk_Ansgd}
In Algorithm~\ref{Alg:NSGD}, let $\eta= R/\Big(L\cdot n\cdot\max\big(\sqrt{n},~ \frac{\sqrt{d\,\log(1/\beta)}}{\alpha}\big)\Big)$.
Then, for any $\theta \in (6\exp(-n/2), 1)$, with probability at least $1-\theta$ over the randomness in both the sample and the algorithm, we have
\begin{align*}
  \risk\left(\A_{\sf NSGD}\right)&= RL\cdot O\left(\max\left(\frac{\log(n)\log(n/\theta)}{\sqrt{n}}, ~\frac{\sqrt{d\,\log(1/\beta)}}{\alpha\, n}\right)\right)
\end{align*}
\end{thm}

\iffull
\begin{proof}
Fix any confidence parameter $\theta\geq 6\exp(-n/2)$. First, for any data set $S\in \Z^n$ and any step size $\eta > 0,$ by Lemma~\ref{lem:online_to_batch} in Appendix~\ref{app:hp-sgd-opt-err}, we have the following high-probability guarantee on the training error of $\Ansgd$: 

\noindent With probability at least $1-\theta/3,$ we have
\begin{align*}
    \eps_{\mbox{\footnotesize opt}}(\Ansgd)\triangleq F_S(\barx)-\min\limits_{x\in\X}F_S(x)&\leq \frac{R^2}{\,\eta\, n^2} +7RL\frac{\sqrt{\log(1/\beta)\log(12/\theta)}}{\alpha n} + \eta\,L^2 \Big(32\frac{d\,\log(1/\beta)}{\alpha^2}+1\Big)
\end{align*}
where the probability is over the sampling in step~\ref{step:sampling} and the independent Gaussian noise vectors $\bG_1, \ldots, \bG_{n^2}$. Given the setting of $\eta$ in the theorem, we get
\begin{align}
   \eps_{\mbox{\footnotesize opt}}(\Ansgd)&\leq 8 R L\max\Big(\frac{1}{\sqrt{n}},~\frac{\sqrt{d\,\log(1/\beta)}}{\alpha\,n}\Big) + 33 R L\,\frac{d\, \frac{\log(1/\beta)}{\alpha^2}}{n\cdot \max\Big(\sqrt{n},~\frac{\sqrt{d\,\log(1/\beta)}}{\alpha}\Big)}\nonumber \\
   &\leq 8 R L\max\Big(\frac{1}{\sqrt{n}},~\frac{\sqrt{d\,\log(1/\beta)}}{\alpha\,n}\Big) + 33 R L\, \frac{\sqrt{d\,\log(1/\beta)}}{n\,\alpha}\nonumber\\
   &= R L\cdot O\Bigg(\max\Big(\frac{1}{\sqrt{n}},~\frac{\sqrt{d\,\log(1/\beta)}}{n\,\alpha}\Big)\Bigg).\label{eqn:Ansgd-opt-err}
\end{align}

Next, it is not hard to show that $\Ansgd$ attains the same UAS bound as $\A_{\sf rSGD}$ (Theorem~\ref{thm:UB_repl_SGD}). Indeed, the only difference is the noise addition in gradient step; however, this does not impact the stability analysis. This is because the sequence of noise vectors $\{\bG_1, \ldots, \bG_{n^2}\}$ is the same for the trajectories corresponding to the pair $S, ~S'$ of neighboring datasets. Hence, the argument basically follows the same lines of the proof of Theorem~\ref{thm:UB_repl_SGD} since the noise terms cancel out. Thus, we conclude that for any pair $S\simeq S'$ of neighboring datasets, with probability at least $1-\exp(n/2)\geq 1-\theta/6$ (over the randomness of $\Ansgd$), the uniform argument stability of $\Ansgd$ is bounded as: $\delta_{\Ansgd}\leq 4L\eta\left(\sqrt{T} + \frac{T}{n}\right),$ where $T=n^2$. Given the setting of $\eta$ in the theorem, this bound reduces to $8 R /\max\big(\sqrt{n},~ \frac{\sqrt{d\,\log(1/\beta)}}{\alpha}\big)$. 

Hence, by Theorem~\ref{thm:stab-to-gen}, with probability at least $1-\theta/3$ (over the randomness in both the i.i.d. dataset $S$ and the algorithm), the generalization error of $\Ansgd$ is bounded as
\begin{align}
    \eps_{\mbox{\footnotesize gen}}(\A_{\sf NSGD})&\leq \frac{8 c\, R L\,\log(n)\log(6n/\theta)}{\max\Big(\sqrt{n},~ \frac{\sqrt{d\,\log(1/\beta)}}{\alpha}\Big)}+\frac{c\,\sqrt{\log(6/\theta)}}{\sqrt{n}}
    = RL \cdot O\left(\frac{\log(n)\log(n/\theta)}{\sqrt{n}}\right),\label{eqn:Ansgd-gen-err}
\end{align}
where $c$ in the first bound is a universal constant. 

Now, using (\ref{eqn:Ansgd-opt-err}), (\ref{eqn:Ansgd-gen-err}), and Lemma~\ref{lem:approx-err-bd}, we finally conlcude that with probability at least $1-\theta$ (over randomness in the sample $S$ and the internal randomness of $\Ansgd$), the excess population risk of $\Ansgd$ is bounded as
\begin{align*}
    \risk(\Ansgd)&\leq \eps_{\mbox{\footnotesize opt}}(\Ansgd) + \eps_{\mbox{\footnotesize gen}}(\A_{\sf NSGD})+\varepsilon_{\mbox{\tiny approx}}\\
    &= R L\cdot O\Bigg(\max\Big(\frac{1}{\sqrt{n}},~\frac{\sqrt{d\,\log(1/\beta)}}{\alpha\, n}\Big) + \frac{\log(n)\log(n/\theta)}{\sqrt{n}}+ \frac{\sqrt{\log(1/\theta)}}{\sqrt{n}}\Bigg)\\
    &= RL\cdot O\Bigg(\max\left(\frac{\log(n)\log(n/\theta)}{\sqrt{n}}, ~\frac{\sqrt{d\,\log(1/\beta)}}{\alpha\, n}\right)\Bigg),
\end{align*}
which completes the proof.
\end{proof}
\fi

\begin{rem}\label{rem:exp_risk_noisy_SGD}
Using the expectation guarantee on UAS given in \iffull Theorem~\ref{thm:UB_repl_SGD} \fi and following similar steps of the analysis above, we can also show that the expected excess population risk of $\Ansgd$ is bounded as: 
$$\ex{}{\risk\left(\Ansgd\right)}= RL\cdot O\left(\max\left(\frac{1}{\sqrt{n}}, ~\frac{\sqrt{d\,\log(1/\beta)}}{\alpha\, n}\right)\right).$$
\end{rem}

\iffull
\subsection{Nonsmooth Stochastic Convex Optimization with Multi-pass SGD}\label{sec:data-resue}

Another application of our results concerns obtaining optimal excess risk for stochastic nonsmooth convex optimization via multi-pass SGD. It is known that one-pass SGD is guaranteed to have optimal excess risk, which can be shown via martingale arguments that trace back to the stochastic approximation literature \citep{Robbins:1951,Kiefer:1952}. 

Using our \UHS bound, we show that Algorithms \ref{alg:repl_SGD} and \ref{alg:perm_SGD} can recover nearly-optimal high-probability excess risk bounds by making $n$ passes over the data.
Analogous bounds hold for Algorithm \ref{alg:batchGD}, however these are less interesting from a computational efficiency perspective.

\ifthenelse{\boolean{fullver}}{
\subsubsection{Risk Bounds for Sampling-with-Replacement Stochastic Gradient Descent}

\begin{thm} \label{thm:gen_repl_SGD}
Let $\X\subseteq\B(0,R)$ and $\F=\F_{\X}^0(L)$. The
sampling with replacement stochastic gradient descent (Algorithm \ref{alg:repl_SGD}) with $T=n^2$ iterations and $\eta=\frac{R}{\,L\,n^{3/2}}$ satisfies for any $12\exp\{-n^2/32\}<\theta<1.$ 
$$\PP \Big[ \varepsilon_{\mbox{\footnotesize{risk}}}(\A_{\sf rSGD}) = O\Big( \frac{cLR}{\sqrt{n}}\log(n)\log(\frac{n}{\theta})\Big) \Big] \leq \theta.$$ 
\end{thm}
It should be noted that, similarly to Remark \ref{rem:exp_risk_noisy_SGD},
if we are only interested in expectation risk bounds, one can shave off the  polylogarithmic factor above, which is optimal for the expected excess risk.

\begin{proof}
Let $\bS\sim{\cal D}^n$ an i.i.d.~random sample for the stochastic convex program, and apply on these data the algorithm $\A_{\sf rSGD}$ for constant step size $\eta>0$ and $T$ iterations.

We consider $\theta>0$ such that $\theta>12\exp\{-T/32\}$.
Notice that the sampling-with-replacement stochastic gradient is a bounded first-order stochastic oracle for the empirical objective. It is direct to verify that the assumptions of Lemma \ref{lem:online_to_batch} are satisfied with $\sigma=0$. 
Hence, by Lemma \ref{lem:online_to_batch}, we have that, with probability at least $1-\theta/3$
$$
\opterr(\A_{\sf rSGD}) \leq O\Big( LR\sqrt{\frac{2\log(12/\theta)}{T}}+\frac{R^2}{\eta T}+\eta L^2 \Big).
$$
On the other hand, Theorem \ref{thm:UB_repl_SGD} together with Theorem \ref{thm:stab-to-gen}, guarantees that
with probability at least $1-\theta/3$, we have
$$ |\epsgen(\A_{\sf rSGD})| \leq O\Big( L^2\big[\sqrt T \eta+\frac{T\eta}{n}\big]\log n\log(6n/\theta)+LR\sqrt{\log(6/
\theta)}{n} \Big).  $$
Finally, Lemma \ref{lem:approx-err-bd} ensures that with probability $1-\theta/3$
$$ \approxerr \leq LR\sqrt{\frac{ 2\log(3/\theta)}{n}}.$$
By the union bound and the excess risk decomposition \eqref{eqn:risk_decomp}, we have that, with probability $1-\theta$,
\begin{eqnarray*} \varepsilon_{\mbox{\footnotesize risk}}({\cal A})
&=& 
O\Big(LR\sqrt{\frac{\log(1/\theta)}{T}}+\frac{R^2}{\eta T}+\eta L^2+L^2\eta \big(\sqrt{T}+\frac{T}{n}\big) \log(n) \log(\frac{6n}{\theta})\\
&&+LR\sqrt{\frac{\log(6/\theta)}{n}}+LR\sqrt{\frac{ \log(3/\theta)}{n}} \Big)\\
&=& O\Big(\frac{LR}{\sqrt n}\log(n)\log(\frac{n}{\theta})\Big),
\end{eqnarray*}
where only at the last step we replaced by the choice of step size and number of iterations from the statement.

\end{proof}

\subsubsection{Risk Bounds for Fixed-Permutation Stochastic Gradient Descent}

As a final application we provide a population risk bound based on the \UHS of Algorithm \ref{alg:perm_SGD}. Similarly as in the case of sampling-with-replacement SGD, we need an optimization error analysis, which for completeness is provided in Appendix \ref{App:Fixed_Perm_SGD}, and it is based on the analysis of the incremental subgradient method \citep{Nedic:2001}.

Interestingly, the combination of the incremental method analysis for arbitrary permutation \citep{Nedic:2001} and our novel stability bounds that also work for arbitrary permutation, guarantees generalization bounds for fixed permutation SGD without the need of reshuffling, or even any form of randomization. We believe this could be of independent interest.



\begin{thm} \label{thm:gen_perm_SGD}
Algorithm \ref{alg:perm_SGD} with constant step size $\eta_k\equiv\eta=R/[Ln\sqrt{K}]$ and $K=n$ epochs is such that for every $0<\theta<1$,
$$ \PP\Big[ \eps_{\mbox{\footnotesize risk}}(\A_{\sf PerSGD}) > \frac{cLR}{\sqrt n} \log n\log(\frac{n}{\theta}) \Big]\leq \theta,$$
where $c>0$ is an absolute constant.
\end{thm}

Similarly to the previous result, we can remove the polylogarithmic factor if we are only interested in expected excess risk guarantees.

\begin{proof}
By Corollary \ref{cor:perm_SGDopt_error}
$$ \eps_{\mbox{\footnotesize opt}}(\A_{\sf PerSGD})
\leq \frac{R^2}{n K\eta}+\frac{L^2(n+2)\eta}{2}=O\Big(\frac{LR}{\sqrt n} \Big),$$
by our choice of $K,\eta$.
On the other hand, Theorem~\ref{thm:UB_perm_SGD} guarantees the algorithm is $\delta$-UAS with probability 1, where $\delta=O(R/\sqrt n)$. Therefore, by
Theorem~\ref{thm:stab-to-gen}, we have that w.p.~$1-\theta/2$
$$|\eps_{\mbox{\footnotesize gen}}(\A_{\sf PerSGD})| \leq c\Big( \frac{LR}{\sqrt n}\log n \log(2n/\theta)+LR\sqrt{\frac{\log 2/\theta}{n}} \Big).
$$
Finally, Lemma \ref{lem:approx-err-bd} ensures that with probability $1-\theta/2$
$$ \approxerr \leq LR\sqrt{\frac{ 2\log(2/\theta)}{n}}.$$
By the union bound and the excess risk decomposition \eqref{eqn:risk_decomp}, we have that, with probability 
at least $1-\theta$,
\begin{eqnarray*}
\eps_{\mbox{\footnotesize risk}}(\A_{\sf PerSGD})
&\leq& \opterr(\A_{\sf PerSGD}) +\generr(\A_{\sf PerSGD})+\approxerr \\
&=& O\Big( \frac{LR}{\sqrt n}
+\frac{LR}{\sqrt n}\log n \log(n/\theta)+LR\sqrt{\frac{\log 1/\theta}{n}} + LR\sqrt{\frac{ \log(2/\theta)}{n}}\Big)\\
&=&O\Big( \frac{LR}{\sqrt n}\log n\log(\frac{n}{\theta}) \Big).
\end{eqnarray*}
\end{proof}

}{In short, we prove the following excess population risk bounds, and their analyses are deferred to Appendix \ref{app:DataReuse},
$$\varepsilon_{\mbox{\footnotesize{risk}}}(\A_{\sf rSGD}) \leq \frac{4RL}{\sqrt{n}} \qquad; \qquad \eps_{\mbox{\footnotesize risk}}(\A_{\sf PerSGD})
\leq \frac{8LR}{\sqrt n}.\vspace{-0.6cm}$$
} 
\fi
\section{Discussion and Open Problems}

In this work we provide sharp upper and lower bounds on \stabname for the (stochastic) subgradient method in stochastic nonsmooth convex optimization. Our lower bounds show inherent limitations of stability bounds compared to the smooth convex case,
however we can still derive optimal population risk bounds by reducing the step size and running the algorithms for longer number of iterations. We provide applications of this idea for differentially-private noisy SGD, and for two versions of SGD (sampling-with-replacement and fixed-permutation SGD).


The first open problem regards lower bounds that are robust to general forms of algorithmic randomization. Unfortunately, the methods presented here are not robust in this respect, since random initialization would prevent the trajectories reaching the region of highly nonsmooth behavior of the objective (or doing it in such a way that it does not increase \UHS). One may try to strengthen the lower bound by using a random rotation of the objective; however, this leads to an uninformative lower bound. Finding distributional constructions for lower bounds against randomization is a very interesting future direction.

Our privacy application provides optimal risk for an algorithm that runs for $n^2$ steps, which is impractical for large datasets. Other algorithms, e.g.~in \citep{FKT19}, run into similar limitations. Proving that quadratic running time is necessary for general nonsmooth DP-SCO is a very interesting open problem 
that can be formalized in terms of the oracle complexity of stochastic convex optimization \citep{nemirovsky1983problem} under stability and/or privacy constraints. 

\subsection*{Acknowledgements}
Part of this work was done while the authors were visiting the Simons Institute for the Theory of Computing during the ``Data Privacy: Foundations and Applications'' program.
RB's research is supported by NSF Awards AF-1908281, SHF-1907715, Google Faculty Research Award, and OSU faculty start-up support. Work by CG was partially funded by the Millennium Science Initiative of the Ministry of Economy,
Development, and Tourism, grant ``Millennium Nucleus Center for the Discovery of Structures in Complex Data.'' CG would like to thank Nicolas Flammarion and Juan Peypouquet for extremely valuable discussions at early stages of this work.

\bibliographystyle{apalike}
\bibliography{vf-allrefs-local,stab-non-smooth-gd}

\begin{thebibliography}{}

\bibitem[Abadi et~al., 2016]{abadi2016deep}
Abadi, M., Chu, A., Goodfellow, I., McMahan, H.~B., Mironov, I., Talwar, K.,
  and Zhang, L. (2016).
\newblock Deep learning with differential privacy.
\newblock In {\em Proceedings of the 2016 ACM SIGSAC Conference on Computer and
  Communications Security}, pages 308--318. ACM.

\bibitem[Bassily et~al., 2019]{BFTT19}
Bassily, R., Feldman, V., Talwar, K., and Thakurta, A.~G. (2019).
\newblock Private stochastic convex optimization with optimal rates.
\newblock In {\em Advances in Neural Information Processing Systems}, pages
  11279--11288.

\bibitem[Bassily et~al., 2014]{BST14}
Bassily, R., Smith, A., and Thakurta, A. (2014).
\newblock Private empirical risk minimization: Efficient algorithms and tight
  error bounds.
\newblock In {\em 2014 IEEE 55th Annual Symposium on Foundations of Computer
  Science (full version available at arXiv:1405.7085)}, pages 464--473. IEEE.

\bibitem[Beck, 2017]{Beck:2017}
Beck, A. (2017).
\newblock {\em First-Order Methods in Optimization}.
\newblock MOS-SIAM Series on Optimization. Society for Industrial and Applied
  Mathematics.

\bibitem[Bousquet and Elisseeff, 2002]{BousquettE02}
Bousquet, O. and Elisseeff, A. (2002).
\newblock Stability and generalization.
\newblock {\em JMLR}, 2:499--526.

\bibitem[Bousquet et~al., 2019]{BousquetKZ19}
Bousquet, O., Klochkov, Y., and Zhivotovskiy, N. (2019).
\newblock Sharper bounds for uniformly stable algorithms.
\newblock {\em CoRR}, abs/1910.07833.

\bibitem[Charles and Papailiopoulos, 2018]{Charles:2018}
Charles, Z. and Papailiopoulos, D. (2018).
\newblock Stability and generalization of learning algorithms that converge to
  global optima.
\newblock In Dy, J. and Krause, A., editors, {\em Proceedings of the 35th
  International Conference on Machine Learning}, volume~80 of {\em Proceedings
  of Machine Learning Research}, pages 745--754, Stockholmsmässan, Stockholm
  Sweden. PMLR.

\bibitem[Chaudhuri and Monteleoni, 2008]{CM08}
Chaudhuri, K. and Monteleoni, C. (2008).
\newblock Privacy-preserving logistic regression.
\newblock In {\em NIPS}.

\bibitem[Chaudhuri et~al., 2011]{CMS}
Chaudhuri, K., Monteleoni, C., and Sarwate, A.~D. (2011).
\newblock Differentially private empirical risk minimization.
\newblock {\em Journal of Machine Learning Research}, 12(Mar):1069--1109.

\bibitem[Chen et~al., 2018]{Chen:2018}
Chen, Y., Jin, C., and Yu, B. (2018).
\newblock Stability and convergence trade-off of iterative optimization
  algorithms.
\newblock {\em CoRR}, abs/1804.01619.

\bibitem[Devroye and Wagner, 1979a]{DevroyeW79}
Devroye, L. and Wagner, T.~J. (1979a).
\newblock Distribution-free inequalities for the deleted and holdout error
  estimates.
\newblock {\em {IEEE} Trans. Information Theory}, 25(2):202--207.

\bibitem[Devroye and Wagner, 1979b]{DevroyeW79a}
Devroye, L. and Wagner, T.~J. (1979b).
\newblock Distribution-free performance bounds with the resubstitution error
  estimate (corresp.).
\newblock {\em {IEEE} Trans. Information Theory}, 25(2):208--210.

\bibitem[Dwork and Feldman, 2018]{DworkFeldman18}
Dwork, C. and Feldman, V. (2018).
\newblock Privacy-preserving prediction.
\newblock {\em CoRR}, abs/1803.10266.
\newblock Extended abstract in COLT 2018.

\bibitem[Dwork et~al., 2006a]{DKMMN06}
Dwork, C., Kenthapadi, K., McSherry, F., Mironov, I., and Naor, M. (2006a).
\newblock Our data, ourselves: Privacy via distributed noise generation.
\newblock In {\em EUROCRYPT}.

\bibitem[Dwork et~al., 2006b]{DMNS06}
Dwork, C., McSherry, F., Nissim, K., and Smith, A. (2006b).
\newblock Calibrating noise to sensitivity in private data analysis.
\newblock In {\em Theory of Cryptography Conference}, pages 265--284. Springer.

\bibitem[Feldman, 2016]{Feldman:16erm}
Feldman, V. (2016).
\newblock Generalization of {ERM} in stochastic convex optimization: The
  dimension strikes back.
\newblock {\em CoRR}, abs/1608.04414.
\newblock Extended abstract in NIPS 2016.

\bibitem[Feldman et~al., 2020]{FKT19}
Feldman, V., Koren, T., and Talwar, K. (2020).
\newblock Private stochastic convex optimization: Optimal rates in linear time.
\newblock Extended abstract in STOC 2020.

\bibitem[Feldman et~al., 2018]{FeldmanMTT18}
Feldman, V., Mironov, I., Talwar, K., and Thakurta, A. (2018).
\newblock Privacy amplification by iteration.
\newblock In {\em {FOCS}}, pages 521--532.

\bibitem[Feldman and Vondr{\'{a}}k, 2018]{FeldmanV:2018}
Feldman, V. and Vondr{\'{a}}k, J. (2018).
\newblock Generalization bounds for uniformly stable algorithms.
\newblock In {\em Advances in Neural Information Processing Systems 31: Annual
  Conference on Neural Information Processing Systems 2018, NeurIPS 2018, 3-8
  December 2018, Montr{\'{e}}al, Canada}, pages 9770--9780.

\bibitem[Feldman and Vondr{\'{a}}k, 2019]{FeldmanV:2019}
Feldman, V. and Vondr{\'{a}}k, J. (2019).
\newblock High probability generalization bounds for uniformly stable
  algorithms with nearly optimal rate.
\newblock In {\em Conference on Learning Theory, {COLT} 2019, 25-28 June 2019,
  Phoenix, AZ, {USA}}, pages 1270--1279.

\bibitem[Hardt et~al., 2016]{HardtRS16}
Hardt, M., Recht, B., and Singer, Y. (2016).
\newblock Train faster, generalize better: Stability of stochastic gradient
  descent.
\newblock In {\em {ICML}}, pages 1225--1234.

\bibitem[Jain et~al., 2012]{jain2012differentially}
Jain, P., Kothari, P., and Thakurta, A. (2012).
\newblock Differentially private online learning.
\newblock In {\em 25th Annual Conference on Learning Theory (COLT)}, pages
  24.1--24.34.

\bibitem[Jain and Thakurta, 2014]{JTOpt13}
Jain, P. and Thakurta, A. (2014).
\newblock (near) dimension independent risk bounds for differentially private
  learning.
\newblock In {\em ICML}.

\bibitem[Kearns and Ron, 1999]{Kearns:1999}
Kearns, M.~J. and Ron, D. (1999).
\newblock Algorithmic stability and sanity-check bounds for leave-one-out
  cross-validation.
\newblock {\em Neural Computation}, 11(6):1427--1453.

\bibitem[Kiefer and Wolfowitz, 1952]{Kiefer:1952}
Kiefer, J. and Wolfowitz, J. (1952).
\newblock Stochastic estimation of the maximum of a regression function.
\newblock {\em Ann. Math. Statist.}, 23(3):462--466.

\bibitem[Kifer et~al., 2012]{kifer2012private}
Kifer, D., Smith, A., and Thakurta, A. (2012).
\newblock Private convex empirical risk minimization and high-dimensional
  regression.
\newblock In {\em Conference on Learning Theory}, pages 25--1.

\bibitem[Koren and Levy, 2015]{KorenLevy15}
Koren, T. and Levy, K. (2015).
\newblock Fast rates for exp-concave empirical risk minimization.
\newblock In {\em {NIPS}}, pages 1477--1485.

\bibitem[Kuzborskij and Lampert, 2018]{Kuzborskij:2018}
Kuzborskij, I. and Lampert, C.~H. (2018).
\newblock Data-dependent stability of stochastic gradient descent.
\newblock In {\em {ICML}}, volume~80 of {\em Proceedings of Machine Learning
  Research}, pages 2820--2829. {PMLR}.

\bibitem[Liu et~al., 2017]{Liu:2017}
Liu, T., Lugosi, G., Neu, G., and Tao, D. (2017).
\newblock Algorithmic stability and hypothesis complexity.
\newblock In {\em Proceedings of the 34th International Conference on Machine
  Learning, {ICML} 2017, Sydney, NSW, Australia, 6-11 August 2017}, pages
  2159--2167.

\bibitem[London, 2017]{London:2017}
London, B. (2017).
\newblock A pac-bayesian analysis of randomized learning with application to
  stochastic gradient descent.
\newblock In Guyon, I., Luxburg, U.~V., Bengio, S., Wallach, H., Fergus, R.,
  Vishwanathan, S., and Garnett, R., editors, {\em Advances in Neural
  Information Processing Systems 30}, pages 2931--2940. Curran Associates, Inc.

\bibitem[Maurer, 2017]{Maurer17}
Maurer, A. (2017).
\newblock A second-order look at stability and generalization.
\newblock In {\em {COLT}}, pages 1461--1475.

\bibitem[Mukherjee et~al., 2006]{MukherjeeNPR06}
Mukherjee, S., Niyogi, P., Poggio, T., and Rifkin, R. (2006).
\newblock Learning theory: stability is sufficient for generalization and
  necessary and sufficient for consistency of empirical risk minimization.
\newblock {\em Advances in Computational Mathematics}, 25(1-3):161--193.

\bibitem[Nedic and Bertsekas, 2001]{Nedic:2001}
Nedic, A. and Bertsekas, D.~P. (2001).
\newblock Incremental subgradient methods for nondifferentiable optimization.
\newblock {\em {SIAM} Journal on Optimization}, 12(1):109--138.

\bibitem[Nemirovsky and Yudin, 1983]{nemirovsky1983problem}
Nemirovsky, A. and Yudin, D. (1983).
\newblock {\em Problem Complexity and Method Efficiency in Optimization}.
\newblock J. Wiley @ Sons, New York.

\bibitem[Poggio et~al., 2004]{PoggioRMN04}
Poggio, T., Rifkin, R., Mukherjee, S., and Niyogi, P. (2004).
\newblock General conditions for predictivity in learning theory.
\newblock {\em Nature}, 428(6981):419--422.

\bibitem[Rigollet, 2015]{mit}
Rigollet, P. (2015.
  ~https://ocw.mit.edu/courses/mathematics/18-s997-high-dimensional-statistics-spring-2015).
\newblock {\em Lecture Notes. 18.S997: High Dimensional Statistics}.
\newblock MIT Courses/Mathematics.

\bibitem[Robbins and Monro, 1951]{Robbins:1951}
Robbins, H. and Monro, S. (1951).
\newblock A stochastic approximation method.
\newblock {\em Ann. Math. Statist.}, 22(3):400--407.

\bibitem[Rogers and Wagner, 1978]{RogersWagner78}
Rogers, W.~H. and Wagner, T.~J. (1978).
\newblock A finite sample distribution-free performance bound for local
  discrimination rules.
\newblock {\em The Annals of Statistics}, 6(3):506--514.

\bibitem[Shalev-Shwartz et~al., 2010]{ShwartzSSS10}
Shalev-Shwartz, S., Shamir, O., Srebro, N., and Sridharan, K. (2010).
\newblock Learnability, stability and uniform convergence.
\newblock {\em The Journal of Machine Learning Research}, 11:2635--2670.

\bibitem[Smith and Thakurta, 2013]{ST13sparse}
Smith, A. and Thakurta, A. (2013).
\newblock Differentially private feature selection via stability arguments, and
  the robustness of the {LASSO}.
\newblock In {\em Conference on Learning Theory (COLT)}, pages 819--850.

\bibitem[Talwar et~al., 2015]{talwar2015nearly}
Talwar, K., Thakurta, A., and Zhang, L. (2015).
\newblock Nearly optimal private {LASSO}.
\newblock In {\em Proceedings of the 28th International Conference on Neural
  Information Processing Systems}, volume~2, pages 3025--3033.

\bibitem[Ullman, 2015]{ullman2015private}
Ullman, J. (2015).
\newblock Private multiplicative weights beyond linear queries.
\newblock In {\em Proceedings of the 34th ACM SIGMOD-SIGACT-SIGAI Symposium on
  Principles of Database Systems}, pages 303--312. ACM.

\bibitem[Wu et~al., 2017]{wu2017bolt}
Wu, X., Li, F., Kumar, A., Chaudhuri, K., Jha, S., and Naughton, J. (2017).
\newblock Bolt-on differential privacy for scalable stochastic gradient
  descent-based analytics.
\newblock In {\em SIGMOD}. ACM.

\bibitem[Zinkevich, 2003]{Zinkevich:2003}
Zinkevich, M. (2003).
\newblock Online convex programming and generalized infinitesimal gradient
  ascent.
\newblock In {\em Proceedings of the Twentieth International Conference on
  International Conference on Machine Learning}, ICML’03, page 928–935.
  AAAI Press.

\end{thebibliography}

\appendix

\ifthenelse{\boolean{fullver}}{}{
\section{Proof of Theorem \ref{thm:UB_GD}} \label{app:pf_UB_GD}
}

\ifthenelse{\boolean{fullver}}{}{
\section{Proof of Theorem \ref{thm:UB_perm_SGD}} \label{app:pf_UB_repl_SGD}

\begin{proof}
The stability bound of $2R$ is implied directly by the diameter of the feasible set. Let $S\simeq S^{\prime}$, and let $(x^t)_{t\in[T]}, (y^t)_{t\in[T]}$ be the trajectories of Algorithm \ref{alg:perm_SGD} on $S$ and $S^{\prime}$, respectively, with $x^1=y^1$.

Notice that since $\boldsymbol{\pi}$ is a random permutation, we may assume w.l.o.g.~that $\boldsymbol{\pi}$ is the identity, whereas the perturbed coordinate between $S,S^{\prime}$ is $\bI\sim\mbox{Unif}([n])$. The rest of the proof is a stability analysis conditioned on $\boldsymbol{\pi}$ (which fixes all the randomness of the algorithm), but from the observation above this is the same as conditioning on the random perturbed coordinate $\bI$. 

Let $T\leq n$, and $\delta_t=\|x^t-y^t\|$ so that $\delta_1=0$. Conditioned on $\bI=i$, we have that for all $t\leq T$,
$$ \delta_{t+1}^2 \leq \left\{
\begin{array}{ll}
0 & t< i \\
4\eta_t^2L^2 & t=i\\
\delta_t^2 + 4\eta_t^2L^2 & i<t\leq T  
\end{array}
\right.$$
Indeed, for all $t\leq i$, $\delta_t=0$. For $t=i$, we have 
\begin{eqnarray*}
\delta_{i+1} &=& \|\proj_{\X}[x^i-\eta_i\nabla f(x^i,z_i)]-\proj_{\X}[y^i-\eta_i\nabla f(y^i,z_i^{\prime})]\|\\
&\leq& \|x^i-y^i-\eta_i[\nabla f(x^i,z_i)-\nabla f(y^i,z_i^{\prime})]\| \\
&\leq& 2L\eta_i,
\end{eqnarray*}
where we used $x^i=y^i$, and that both gradients are bounded in norm by $L$.
Finally, when $t>i$, we have
$z_t=z_t^{\prime}$, and therefore we can leverage the monotonicity
of the subgradients
\begin{eqnarray*}
\delta_{t+1}^2 &=& \|\proj_{\X}[x^t-\eta_t\nabla f(x^t,z_t)]-\proj_{\X}[y^t-\eta_t\nabla f(y^t,z_t)]\|^2 \\
&\leq& \delta_t^2+4L^2\eta_t^2-2\eta_t\langle \nabla f(x^t,z_t)-\nabla f(y^t,z_t),x^t-y^t \rangle\\
&\leq& \delta_t^2+4L^2\eta_t^2.
\end{eqnarray*}
Unravelling this recursion, we get $\mathbb{E}[\delta_{t+1}^2| \bI=i] \leq 4L^2\sum_{s=i}^{t} \eta_s$, so by the law of total expectation:
\begin{eqnarray*}
\EE[\delta_t] &=& \frac1n\sum_{i=1}^n \EE[\delta_t|\bI=i]
\leq \frac1n\sum_{i=1}^n \sqrt{\EE[\delta_t^2|\bI=i]}
\leq \frac{2L}{n}\sum_{i=1}^{t-1} \sqrt{\sum_{s=i}^{t-1}
\eta_s^2}\\
&\leq& \frac{2L}{n}\sum_{i=1}^{t-1} \sqrt{(t-i)}\eta_i
\leq \frac{2L}{n}
\sqrt{\Big(\sum_{i=1}^{t-1}(t-i)\Big)\Big(\sum_{i=1}^{t-1}\eta_i^2}\Big)\\
&\leq& \frac{2L}{n}
\sqrt{\frac{(t-1)^2}{2}\sum_{i=1}^{t-1}\eta_i^2}
=\frac{\sqrt{2}L(t-1)}{n}
\sqrt{\sum_{i=1}^{t-1}\eta_i^2}.
\end{eqnarray*}
where the first inequality holds by the Jensen inequality, the second inequality comes from the bound on the conditional expectation, the third inequality from the non-increasing stepsize assumption, and the fourth inequality is from Cauchy-Schwarz. 
Since averaging can only improve stability, we conclude the result.
\end{proof}}

\ifthenelse{\boolean{fullver}}{}{
\section{General Lower Bound on Stability of SGD}
\label{app:gen_LB_FOM}

Let ${\cal A}$ be a $\gamma$-uniformly stable stochastic convex optimization algorithm  with  $\gamma=s(T)/n$, where $s(T)$ is increasing and $\lim_{T\rightarrow +\infty}s(T)=+\infty$. By the lower bound on the optimal risk of nonsmooth convex optimization, $\epsrisk\geq \frac{LR}{C_1 \sqrt{n}}$, where $C_1>0$ is a universal constant \citep{nemirovsky1983problem}. This, combined with the risk decomposition \eqref{eqn:risk_decomp}, implies that
$$\epsopt\geq \frac{LR}{C_1\sqrt n}-\frac{s(T)}{n}
= -s(T)\Big( \frac{1}{\sqrt n}-\frac{LR}{2C_1 s(T)} \Big)^2+\frac{(LR)^2}{4C_1^2 s(T)}.$$
By our assumption on $s(T)$, for $T$ sufficiently large, there always exists $n$ such that
$$ \frac{4C_1s(T)}{3LR} \leq \sqrt{n} \leq  \frac{4C_1s(T)}{LR}  $$
which leads to $\epsopt\geq \frac{(LR)^2}{C_2 s(T)}$, where $C_2>0$ is a universal constant.

If algorithm ${\cal A}$ is based on $T$ subgradient iterations with constant step size $\eta>0$ (these could be either stochastic, batch or minibatch), by standard analysis, the optimization guarantee of such algorithm is $\epsopt\leq\frac12(\frac{R^2}{\eta T}+\eta L^2)$. Both bounds in combination give
$$ s(T) \geq \frac{2(LR)^2}{C_2 (\eta L^2+R^2/[\eta T])}
=\frac{2(LR)^2\eta T}{C_2(\eta^2TL^2+R^2)}. $$
If we further assume that $\eta \leq (R/L)/\sqrt{T}$ (notice $\eta= (R/L)/\sqrt{T}$ minimizes the optimization error), then $s(T)\geq L^2\eta T/C_2$. We also emphasize all the choices of step size that we will make to control generalization error will lie in this range. 
}

\ifthenelse{\boolean{fullver}}{}{
\section{Proof of Theorem \ref{thm:LB_GD_basic}}
\label{app:LB_GD_basic}

}

\ifthenelse{\boolean{fullver}}{}{
\section{Proof of Theorem \ref{thm:LB_perm_SGD_basic}}
\label{app:LB_perm_SGD_basic}

\begin{proof}
We consider the same function class of Thm.~\ref{thm:LB_GD_basic}, and neighbor datasets $S^{\prime}=(0,0,\ldots,0)$, $S=(1,0,\ldots,0)$. We will assume in what follows that $D=\min\{T,1/\eta^2\}$, $K$ is sufficiently large and $\nu<\eta \|r\|/K$. 
Let $(x^t)_{t\in[T]}$ and $(y^t)_{t\in[T]}$ be the trajectories of Algorithm \ref{alg:perm_SGD} over datasets $S,S^{\prime}$ respectively, both initialized at $x^1=y^1=0$.
Let now $\tau=\boldsymbol{\pi}^{-1}(1)\sim\Unif[n]$. Arguing as in Thm.~\ref{thm:LB_GD_basic}, we have that $y^t=0$ for all $t$, whereas
$$
x^{t+1} = 
\left\{ 
\begin{array}{ll}
0 & t< \tau\\
-\frac{\eta (1+\lfloor t/n\rfloor) r}{K}-\eta\sum_{s=1}^{t-\tau-(1+\lfloor t/n\rfloor)} e_{s} & \tau\leq t\leq \tau+D.
\end{array}
\right.
$$
Later iterations 
will satisfy $\|x^t\|=1-o(1)$ if $D=1/\eta^2$ (and otherwise the algorithm stops earlier). 
Therefore, for all $t\in [T]$,
\begin{eqnarray*}
\EE_{\boldsymbol{\pi}}[\|x^t-y^t\|]
&=&\sum_{s=1}^n\EE[\|x^t-y^t\|\,|\,\tau=s]\PP[\tau=s]\\
&\geq&\frac{\eta}{2n}\sum_{s=1}^{\min\{t,n\}} \sqrt{t-s} - \eta\cdot o(1)\\
&=&
\left\{
\begin{array}{ll}
\Omega(\frac{\eta t^{3/2}}{n}) & \mbox{ if } t\leq n\\
\Omega(\eta\sqrt t) &\mbox{ if } t>n. 
\end{array}
\right.
\end{eqnarray*}
Notice that we used above that $K$ is such that $T\sqrt{D}/nK=o(1)$. 
Analogously as in Thm.~\ref{thm:LB_GD_basic}, we can obtain the same conclusion for $\barx^T$. The lower bound of $\eta T/n$ can be added by \ifthenelse{\boolean{fullver}}{Observation \ref{obs:base_lb}}{Appendix \ref{app:gen_LB_FOM}}, so the result follows.
\end{proof}
}

\section{Upper bounds on UAS of SGD when $T\leq n$} \label{app:T_less_n}

\begin{thm} \label{thm:UB_repl_SGD_2}
Let $\X\subseteq\B(0,R)$ and $\F=\F_{\X}^0(L)$. Suppose $T\leq n$. The UAS of sampling-with-replacement stochastic gradient descent (Algorithm \ref{alg:repl_SGD}) satisfies \stabname
$$\EE\left[\delta_{\A_{\sf rSGD}}\right]\leq \min\left(2R,~ 3L\,\frac{T-1}{n}\,\left(\sqrt{\sum_{t=1}^{T-1}\eta_t^2}  +\frac{1}{n}\sum_{t=1}^{T-1}\eta_t\right)\right).$$
\end{thm}

\begin{proof}
The bound of $2R$ is obtained directly from the diameter bound on $\X$. Therefore, we focus exclusively on the second term. Let $S\simeq S^{\prime}$, and let $k\in[n]$ be the entry where both datasets differ. Let $(x^t)_{t\in[T]}, (y^t)_{t\in[T]}$ be the trajectories of Algorithm \ref{alg:repl_SGD} on $S$ and $S^{\prime}$, respectively, with $x^1=y^1$.

Let $B_t$ denote the event that $\bI_j = k$ for some $j\leq t$; that is, $B_t$ is the event that the index $k$ is sampled at least once in the first $t$ iterations. We note that
\ifthenelse{\boolean{fullver}}{
\begin{eqnarray*}
    \pr{}{B_t}&=&\frac{1}{n}\sum_{j=0}^{t-1}\big(1-\frac{1}{n}\big)^{j}=1-\big(1-\frac{1}{n}\big)^t\leq \frac{t}{n}.
\end{eqnarray*}
Hence, we have
\begin{align}
\ex{}{\delta_{T}}\leq \frac{T-1}{n}\cdot \ex{}{\delta_{T} \vert~B_{T-1}}\label{ineq:cond_delta_B_T}.
\end{align}
}{
\begin{eqnarray}
\pr{}{B_t}&=&\frac{1}{n}\sum_{j=0}^{t-1}\big(1-\frac{1}{n}\big)^{j}=1-\big(1-\frac{1}{n}\big)^t\leq \min\left(1,~\frac{t}{n}\right) \notag\\
\Longrightarrow \quad
\ex{}{\delta_{T}}&\leq& \min\left(1,~\frac{T-1}{n}\right)\cdot \ex{}{\delta_{T} \vert~B_{T-1}}\label{ineq:cond_delta_B_T}.
\end{eqnarray}
}

For the rest of the proof we bound $\ex{}{\delta_T\vert B_{T-1}}$. To do this, we derive a recurrence for $\ex{}{\delta_{t+1}\vert B_t}$. Note that $B_t$ is the union of two mutually exclusive events: $\{\bI_t=k\}\cap\overline{B_{t-1}}$ and $B_{t-1}$, where $\overline{B_{t-1}}$ is the complement of $B_{t-1}$, i.e., 
the event ``index $k$ is never sampled in the first $t$ iterations.'' Hence, 
\begin{eqnarray}
\ex{}{\delta^2_{t+1}~\vert~B_t}&=&\pr{}{\bI_t=k,~\overline{B_{t-1}}~\vert ~B_t}\ex{}{\delta^2_{t+1}~\vert~\bI_t=k,~\overline{B_{t-1}}}+\pr{}{B_{t-1}~\vert~B_t}\ex{}{\delta^2_{t+1}~\vert~B_{t-1}}\nonumber\\
&\leq& \ex{}{\delta^2_{t+1}~\vert~\bI_t=k,~\overline{B_{t-1}}}+\ex{}{\delta^2_{t+1}~\vert~B_{t-1}}.\label{ineq:decompose_B_t}
\end{eqnarray}
%
%
Now, conditioned on the past sampled coordinates $\bI_1,\ldots,\bI_{t-1}$, we have 
\begin{eqnarray*}
\delta_{t+1} &=& \|\proj_{\X}[x^t-\eta_t\nabla f(x^t,z_{\bI_t})]-\proj_{\X}[y^t-\eta_t\nabla f(y^t,z_{\bI_t}^{\prime})]\|\\
&\leq&\|x^t-y^t-\eta_t[\nabla f(x^t,z_{\bI_t})-\nabla f(y^t,z_{\bI_t}^{\prime})]\|\\
&\leq&
\mathbf{1}_{\{\bI_{t} = k \}}(\delta_t+2L\eta_t)+
\mathbf{1}_{\{\bI_{t} \neq k \}}\sqrt{\delta_t^2+4L^2\eta_t^2},
\end{eqnarray*}
where the last inequality is obtained from convexity and Lipschitzness of the objective.
Now, squaring we get
\begin{eqnarray*}
\delta_{t+1}^2&\leq&
\mathbf{1}_{\{\bI_{t} = k \}}(\delta_t+2L\eta_t)^2+
\mathbf{1}_{\{\bI_{t} \neq k \}}(\delta_t^2+4L^2\eta_t^2)
\ifthenelse{\boolean{fullver}}{\\
&=&}{\,\,=\,\,}
\delta_t^2 + 4\eta_t^2 L^2+
\mathbf{1}_{\{\bI_{t} = k \}} 4L\eta_t \delta_t.
\end{eqnarray*}
From this formula we derive bounds for the two conditional expectations:
\begin{eqnarray}
\ex{}{\delta_{t+1}^2~\vert~B_{t-1}}&\leq&
\ex{}{\delta_t^2~\vert~B_{t-1} } +4L^2\eta_t^2 + \frac{4L}{n}\eta_t \ex{}{\delta_t~\vert~ B_{t-1}} \label{eqn:cond_bd_1}
\\
 \ex{}{\delta^2_{t+1}~\vert~\bI_t=k,~\overline{B_{t-1}}} &\leq& 4L^2\eta_t^2, \label{eqn:cond_bd_2}
\end{eqnarray}
where \eqref{eqn:cond_bd_1} holds by independence of $\bI_{t}$ and $B_{t-1}$, and in \eqref{eqn:cond_bd_2} we used that $\delta_t=0$ conditioned on $\overline{B_{t-1}}$.

Combining \eqref{eqn:cond_bd_1} and \eqref{eqn:cond_bd_2} in \eqref{ineq:decompose_B_t}, we get
\begin{eqnarray*}
\ex{}{\delta_{t+1}^2~\vert~ B_t} &\leq&
\ex{}{\delta_{t}^2~\vert~ B_{t-1}} + 8L^2\eta_t^2 + \frac{4L}{n}\eta_t\ex{}{\delta_t~\vert~B_{t-1}}\\
\\
\EE[\delta_{T}^2 ~|~ B_{T-1}] &\leq&
8L^2\sum_{t=1}^{T-1}\eta_t^2 +\frac{4L}{n}\sum_{t=1}^{T-1}\eta_t\EE[\delta_t~| B_{t-1}].
\end{eqnarray*}
With this last bound we can proceed inductively to show that
$$\ex{}{\delta_{T}~|~B_{T-1}}\leq \sqrt{8}L \sqrt{\sum_{t=1}^{T-1}\eta_t^2}  +\frac{2L}{n}\sum_{t=1}^{T-1}\eta_t.$$
The base case, $T=0$, is evident, and the inductive step can be considered in two separate cases; namely, the case where  $\EE[\delta_T~|~B_{T-1}]\leq \max_{t\in[T-1]} \EE[\delta_t~|~B_{t-1}]$, which can be obtained by the induction hypothesis; and the case where $\EE[\delta_T~|~B_{T-1}]> \max_{t\in[T-1]} \EE[\delta_t~|~B_{t-1}]$, for which
\begin{eqnarray*}
\EE[\delta_{T}^2|B_{T-1}] \! &\leq& \!
8L^2\sum_{t=1}^{T-1}\eta_t^2 +\frac{4L}{n}\sum_{t=1}^{T-1}\eta_t\EE[\delta_t|B_{t-1}]
\,\,\leq\,\, 8L^2\sum_{t=1}^{T-1}\eta_t^2 +\frac{4L}{n}\sum_{t=1}^{T-1}\eta_t\EE[\delta_T|B_{T-1}].
\end{eqnarray*}
Then
\begin{eqnarray*}
\EE_{\bI}\Big[\Big(\delta_T-\frac{2L}{n}\sum_{t=1}^{T-1}\eta_t\Big)^2~\Big\vert~B_{T-1}\Big]
&\leq& 8L^2\sum_{t=1}^{T-1}\eta_t^2  +\Big(\frac{2L}{n}\sum_{t=1}^{T-1}\eta_t \Big)^2,
\end{eqnarray*}
and by the Jensen inequality
\begin{eqnarray*}
\EE_{\bI}\Big[\delta_T-\frac{2L}{n}\sum_{t=1}^{T-1}\eta_t\Big\vert B_{T-1}\Big]
&\leq& \sqrt{8L^2\sum_{t=1}^{T-1}\eta_t^2  +\Big(\frac{2L}{n}\sum_{t=1}^{T-1}\eta_t \Big)^2}
\,\,\leq\,\, \sqrt{8}L\sqrt{\sum_{t=1}^{T-1}\eta_t^2} +\frac{2L}{n}\sum_{t=1}^{T-1}\eta_t,
\end{eqnarray*}
proving the inductive step. Finally, putting this together with (\ref{ineq:cond_delta_B_T}) completes the proof.
\end{proof}


\begin{thm} \label{thm:UB_perm_SGD_2}
Let $\X\subseteq\B(0,R)$, $\F=\F_{\X}^0(L)$, $\boldsymbol{\pi}$ be a uniformly random permutation over $[n]$, and $(\eta_t)_{t\in[T]}$ be a non-increasing sequence. Let $T\leq n$. The UAS of fixed-permutation stochastic gradient descent (Algorithm \ref{alg:perm_SGD}) satisfies \stabname
$\EE\left[\delta_{\A_{\sf PerSGD}}\right]\leq \min\{2R,  \sqrt{2}L\frac{T-1}{n}\sqrt{\sum_{t=1}^{T-1} \eta_t^2}\}$. 
\end{thm}
\ifthenelse{\boolean{fullver}}{
\begin{obs} \label{obs:equiv_rates_SGD}
Note that the bound above for fixed permutation SGD in Theorem~\ref{thm:UB_perm_SGD_2} is of the same order as that of sampling with replacement SGD in Theorem~\ref{thm:UB_repl_SGD_2}.
This is because $\sqrt{\sum_{t=1}^{T-1}\eta_t^2}\geq \frac{1}{\sqrt{T}}\sum_{t=1}^{T-1}\eta_t$ (by Cauchy-Schwarz inequality), and hence, when $T\leq n$, we would have $\sqrt{\sum_{t=1}^{T-1}\eta_t^2}\geq \frac{1}{\sqrt{n}}\sum_{t=1}^{T-1}\eta_t\geq \frac{1}{n}\sum_{t=1}^{T-1}\eta_t$.
\end{obs}
}{
Notice that the \UHS of  Algorithms~\ref{alg:repl_SGD} and \ref{alg:perm_SGD} are of the same order. The proof is in Appendix \ref{app:pf_UB_repl_SGD}.
}


\ifthenelse{\boolean{fullver}}{

}{}


\ifthenelse{\boolean{fullver}}{}{
\section{Generalization in Stochastic Optimization with Data Reuse}
\label{app:DataReuse}

}

\section{High-probability Bound on Optimization Error of SGD with Noisy Gradient Oracle}\label{app:hp-sgd-opt-err}

It is known that standard online-to-batch conversion technique can be used to provide high-probability bound on the optimization error (i.e., the excess empirical risk) of stochastic gradient descent. For the sake of completeness and to make the paper more self-contained, we re-expose this technique here for stochastic gradient descent with noisy gradient oracle. This is done in the following lemma, which is used in the proofs of our results in Section~\ref{sec:Applications}.


\begin{lem}[Optimization error of SGD with noisy gradient oracle] \label{lem:online_to_batch}
Let $S=(z_1,\ldots,z_n)\in \Z^n$ be a dataset. Let $F_{S}(x)=\frac1n\sum_{i\in[n]}f(x,z_i)$ be the empirical risk associated with $S$, where for every $z\in\Z,$ $f(\cdot, z)$ is convex, $L$-Lipschitz function over $\X\subseteq \B(0,R)$ for some $L, R >0$. Consider the stochastic (sub)gradient method:
$$ x^{t+1}=x^t-\eta\cdot\bg(x,\xi_t) \qquad(\forall t=0,\ldots,T-1), $$
with output $\overline{x}^T=\frac1T\sum_{t\in[T]}x^t$; where $\xi_1,\ldots,\xi_T$ are drawn uniformly from  from $S$ with replacement, and for all $z\in\Z$, $\bg(., z):\X\rightarrow \RR^d$ is a random map (referred to as noisy gradient oracle) that satisfies the following conditions:
\begin{enumerate}
    \item Unbiasedness: For every $x\in\X, z\in\Z$,~ $\EE[\bg(x,z)]=\nabla f(x, z)$, where the expectation is taken over the randomness in the gradient oracle $\bg(\cdot, z)$.
    \item Sub-Gaussian gradient noise: There is $\sigma^2\geq0$ such that for every $x\in\X, z\in\Z$, $\bg(x, z)-\nabla f(x, z)$ is $\sigma^2$-sub-Gaussian random vector; that is, for every $x\in\X, z\in\Z,$ $\langle \bg(x, z)- \nabla f(x, z),~ u\rangle$ is $\sigma^2$-sub-Gaussian random variable $\forall u\in \B(0, 1)$.
    \item Independence of the gradient noise across iterations: conditioned on any fixed realization of $(\xi_t)_{t\in [T]}$ the sequence of random maps $\bg(\cdot, \xi_1), \ldots, \bg(\cdot, \xi_T)$ is independent. (Here, randomness comes only from the gradient oracle.)
\end{enumerate}
Then, for any $\theta\in (4e^{-T/32}, 1)$, with probability at least $1-\theta$, the optimization error (i.e., the excess empirical risk) of this method is bounded as 
$$\opterr \leq \left(LR+\sigma(R+\eta L)\right)\sqrt{\frac{2\log(4/\theta)}{T}}+\frac{R^2}{2\eta T}+
 \eta \,(\frac{L^2}{2} +d\sigma^2).$$
\end{lem}

\begin{proof}
Let $x^{\ast}_S\in \arg\min\limits_{x\in \X}F_S(x)$. By convexity of the empirical loss, we have
\begin{align} 
 &F_{S}(\overline{x}^T) -F_S(x^{\ast}_S)
\leq \frac{1}{T} \sum_{t\in[T]} F_{S}(x^t)-F_S(x^{\ast}_S)\nonumber\\
&=   \frac{1}{T}\sum_{t\in[T]} [F_{S}(x^t)-f(x^t,\xi_t)]+ \frac{1}{T}\sum_{t\in[T]}[f(x^{\ast}_S,\xi_t)-F_{S}(x^{\ast}_S)]+
\frac{1}{T}\sum_{t\in[T]} [f(x^t,\xi_t)-f(x^{\ast}_S,\xi_t)]. \label{eqn:online_to_batch}
\end{align}
Since $(\xi_t)_{t\in [T]}$ are sampled uniformly without replacement from $S,$ we have $$\ex{\xi_t~|~x^1,\ldots, x^{t-1}}{f(x^t, \xi_t)|x^{1}, \ldots, x^{t-1}, x^{t}=v}=F_S(v),$$ 
for all $v\in\X, t\in [T]$. Moreover, since the range of $f$ lies in $[-LR, LR],$ it follows that $Y_t:=\sum_{j=1}^t f(x^{j}, \xi_j), ~t\in [T]$ is a martingale with bounded differences (namely, bounded by $2LR$). Therefore, by Azuma's inequality, the first term in \eqref{eqn:online_to_batch} satisfies
\begin{align}
\PP \Bigg[\frac{1}{T}\sum_{t\in[T]} [F_{S}(x^t)-f(x^t,\xi_t)] >
LR\sqrt{\frac{2\log\frac{4}{\theta}}{T}}~\Bigg] \leq \frac{\theta}{4}.\label{ineq:azuma}
\end{align}

By Hoeffding's inequality, the second term in \eqref{eqn:online_to_batch} also satisfies the same bound; namely, 
\begin{align}
\PP \Bigg[\frac{1}{T}\sum_{t\in[T]}[f(x^{\ast}_S, \xi_t)-F_{S}(x^{\ast}_S)] >
LR\sqrt{\frac{2\log\frac{4}{\theta}}{T}}~\Bigg] \leq \frac{\theta}{4}.\label{ineq:hoeff}
\end{align}
Using similar analysis to that of the standard online gradient descent analysis \cite{Zinkevich:2003},
the last term in \eqref{eqn:online_to_batch} can be bounded as
\begin{align} 
 &\frac{1}{T}\sum_{t\in[T]}[f(x^t,\xi_t)-f(x^{\ast}_S,\xi_t)] \leq \frac{R^2}{2T\eta} +\frac{1}{T}\sum_{t\in[T]}\langle \nabla f(x^t,\xi_t)-\bg(x^t, \xi_t),~ x^t-x^{\ast}_S\rangle
+\frac{\eta}{2T} \sum_{t\in[T]}\|\nabla \bg(x^t,\xi_t)\|^2\nonumber\\
=~&~\frac{R^2}{2T\eta} +\frac{1}{T}\sum_{t\in[T]}\langle \nabla f(x^t,\xi_t)-\bg(x^t, \xi_t),~ x^t-x^{\ast}_S -\eta \nabla f(x^t, \xi_t)\rangle +\frac{\eta}{2T} \sum_{t\in[T]}\|\bg(x^t,\xi_t)-\nabla f(x^t, \xi_t)\|^2\nonumber\\
&+\frac{\eta}{2T} \sum_{t\in[T]}\|\nabla f(x^t, \xi_t)\|^2\nonumber\\
\leq~&~ \frac{R^2}{2T\eta}+\frac{\eta L^2}{2} +\frac{1}{T}\sum_{t\in[T]}\langle \nabla f(x^t,\xi_t)-\bg(x^t, \xi_t),~ x^t-x^{\ast}_S -\eta \nabla f(x^t, \xi_t)\rangle
+\frac{\eta}{2T} \sum_{t\in[T]}\|\bg(x^t,\xi_t)-\nabla f(x^t, \xi_t)\|^2\label{eqn:reg_bd}
\end{align}
By the properties of the gradient oracle stated in the lemma, we can see that for any fixed realization of $(x^t, \xi_t)_{t\in [T]}$, the second term in \eqref{eqn:reg_bd} is $\left(2R+\eta L\right)^2\frac{\sigma^2}{T}$-sub-Gaussian random variable. Hence,
\begin{align}
\pr{}{\frac{1}{T}\sum_{t\in[T]}\langle \nabla f(x^t,\xi_t)-\bg(x^t, \xi_t), ~x^t-x^{\ast}_S -\eta \nabla f(x^t, \xi_t)\rangle > \left(2R+\eta L\right)\sigma\sqrt{\frac{2\log(4/\theta)}{T}}}&\leq \frac{\theta}{4}.\label{ineq:subgauss1}
\end{align}
Let $U_t:=\|\bg(x^t,\xi_t)-\nabla f(x^t, \xi_t)\|^2,~ t\in [T]$. Note that $\ex{}{U_t}\leq d\sigma^2$. Moreover, observe (e.g., by \cite[Lemma~1.12]{mit}) that for any fixed realization of $x^t, \xi_t$, $V_t:=U_t-\ex{}{U_t}$ is a sub-exponential random variable with parameter $16d\sigma^2$; namely, $\ex{}{\exp(\lambda V_t)}\leq \exp(128\lambda^2\sigma^4d^2),~ \lambda \leq \frac{1}{16\sigma^2 d}.$ 
Hence, by a standard concentration argument (e.g., Bernstein's inequality), we have 
\begin{align}
\pr{}{\frac{\eta}{2T} \sum_{t\in[T]}\|\bg(x^t,\xi_t)-\nabla f(x^t, \xi_t)\|^2> \frac{\eta}{2}d\sigma^2 + 16\eta d\sigma^2\,\frac{\log(4/\theta)}{T}}&\leq \theta/4.\label{ineq:subgauss2}
\end{align}
Putting \eqref{ineq:subgauss1} and \eqref{ineq:subgauss2} together, and noticing that $T>32\log(4/\theta)$, we conclude that with probability at least $1-\theta/2,$ the third term of \eqref{eqn:online_to_batch} is bounded as $$\frac{1}{T}\sum_{t\in[T]}[f(x^t,\xi_t)-f(x^{\ast}_S,\xi_t)]\leq \frac{R^2}{2T\eta}+\frac{\eta L^2}{2}+\eta\sigma^2 d+\left(2R+\eta L\right)\sigma\sqrt{\frac{2\log(4/\theta)}{T}}.$$ 
 
Hence, by the union bound, we conclude that with probability at least $1-\theta,$ the excess empirical risk of the stochastic subgradient method is bounded as 

$$\opterr \leq \left(LR+\sigma(2R+\eta L)\right)\sqrt{\frac{2\log(4/\theta)}{T}}+\frac{R^2}{2\eta T}+
 \eta \,(\frac{L^2}{2} +d\sigma^2).$$
 
\end{proof}

\section{Empirical Risk of Fixed-Permutation SGD} \label{App:Fixed_Perm_SGD}

Our optimization error analysis is based on \cite[Lemma 2.1]{Nedic:2001}.
\begin{lem} \label{lem:perm_SGD_round}
Let us consider the fixed permutation stochastic gradient descent (Algorithm \ref{alg:perm_SGD}), for arbitrary permutation (i.e., not necessarily random) and with constant step size over each epoch (i.e., $\eta_{(k-1)n+t}\equiv\eta_k$ for all $t\in[n]$, $k\in[K]$). Then
$$
\eta_k[F_S(x^k)-F_{S}(y)]
\leq \frac{1}{2n}[\|x^k-y\|^2-\|x^{k+1}-y\|^2] + \frac{\eta_k^2L^2(n+2)}{2}
\qquad(\forall y\in \X).
$$
\end{lem}

\begin{proof}
First, since the permutation is arbitrary, 
w.l.o.g.~$\boldsymbol{\pi}$ is the identity (we make this choice only for
notational convenience).
Let now $y\in\X$. At each round, the recursion of SGD implies that
\begin{eqnarray*}
\|x_{t+1}^k-y\|^2
&=& \|\proj_{\X}[x_{t}^k-\eta_k \nabla f(x_t^k,z_{t})]-\proj_{\X}(y)\|^2 \\
&\leq&\|x_{t}^k-\eta_k \nabla f(x_t^k,z_t)-y\|^2 \\
&=&\|x_t^k-y\|^2+\eta_k^2L^2 -2\eta_k\langle \nabla f(x_t^k,z_t),x_t^k-y\rangle \\
&\leq& \|x_t^k-y\|^2+\eta_k^2L^2
-2\eta_k[f(x_t^k,z_t)-f(y,z_t)].
\end{eqnarray*}
Let $r_t:=\|x^t-y\|$.
Summing up these inequalities from $t=1,
\ldots,n$
\begin{eqnarray*}
r_{n+1}^2-r_1^2
&\leq&
nL^2\eta_k^2+2\eta_k \sum_{t=1}^n[f(x^k,z_t)-f(x_t^k,z_t)]
-2\eta_k n[F_S(x^k)-F_S(y)]\\
&\leq& nL^2\eta_k^2+2\eta_k \sum_{t=1}^n
L\|x^k-x_t^k\|
-2\eta_k n[F_S(x^k)-F_S(y)]\\
&\leq& nL^2\eta_k^2
+2\eta_k^2L^2 \sum_{t=1}^n t
-2\eta_k n[F_S(x^k)-F_S(y)]\\
&=& \eta_k^2 L^2 n
+\eta_k^2L^2n(n+1)
-2\eta_k n[F_S(x^k)-F_S(y)].
\end{eqnarray*}
Re-arranging terms we obtain the result.
\end{proof}

Using the previous lemma, it is straightforward to derive the optimization accuracy of the method.

\begin{cor} \label{cor:perm_SGDopt_error}
The fixed permutation stochastic gradient descent (Algorithm \ref{alg:perm_SGD}), for arbitrary permutation (i.e., not necessarily random) and with constant step size over each epoch (i.e., $\eta_{(k-1)n+t}\equiv\eta_k$ for all $t\in[n]$, $k\in[K]$).
satisfies
$$ \epsopt
\leq \frac{\|x^1-x^{\ast}(S)\|^2}{2n\sum_k \eta_k}+\frac{L^2(n+2)}{2}\frac{\sum_k \eta_k^2}{\sum_k \eta_k}.$$
\end{cor}

\begin{proof}
By convexity and
Lemma \ref{lem:perm_SGD_round}, we have
\begin{eqnarray*}
F_s(\overline{x}^K)-F_{S}(x^{\ast}(S))
&\leq& \frac{1}{\sum_{k=1}^K \eta_k}
\sum_{k=1}^K\eta_k[F_S(x^k)-f_S(x^{\ast}(S))]\\
&\leq& \frac{1}{\sum_{k=1}^K \eta_k}
\Big[ \frac{1}{2n}\|x^1-x^{\ast}(S)\|^2
+\frac{L^2(n+2)}{2}\sum_{k=1}^K\eta_k^2\Big],
\end{eqnarray*}
which proves the result.
\end{proof} 
\end{document}